\documentclass[10pt,twoside,journal]{IEEEtran}
\ifCLASSINFOpdf
\else
\fi
\usepackage[brazil, english]{babel}
\usepackage{hyperref}
\usepackage{graphicx}
\usepackage{amsmath, amssymb, amsthm, bm, algorithm, algorithmic}
\usepackage{color, url, balance, times, helvet, courier}
\usepackage{booktabs, threeparttable}
\usepackage[square,comma,sort&compress,numbers]{natbib}
\usepackage{threeparttable,multirow,adjustbox,threeparttable}

\newtheorem{theorem}{Theorem}
\newtheorem{prop}{Proposition}
\newtheorem{defn}{Definition}
\theoremstyle{remark}
\newtheorem{remark}{Remark}
\usepackage{changes}
\colorlet{Changes@Color}{red}

\begin{document}
\title{Structured Sparse Non-negative Matrix Factorization with $\ell_{2,0}$-Norm for scRNA-seq Data Analysis}

\author{Wenwen Min, Taosheng Xu, Xiang Wan and Tsung-Hui Chang
\IEEEcompsocitemizethanks{
\IEEEcompsocthanksitem Wenwen Min is with The Chinese University of Hong Kong, Shenzhen 518172, China, University of Science and Technology of China, Hefei 230027, China, and the Shenzhen Research Institute of Big Data, Shenzhen 518172, China. E-mail: minwenwen@ustc.edu.cn.
\IEEEcompsocthanksitem Taosheng Xu is with Warshel Institute for Computational Biology, The Chinese University of Hong Kong, Shenzhen 518172, China, University of Science and Technology of China, Hefei 230027, China. E-mail: taosheng.x@gmail.com
\IEEEcompsocthanksitem Xiang Wan is with the Shenzhen Research Institute of Big Data, Shenzhen 518172, China. E-mail: wanxiang@sribd.cn.
\IEEEcompsocthanksitem Tsung-Hui Chang is with The Chinese University of Hong Kong, Shenzhen 518172, China and the Shenzhen Research Institute of Big Data, Shenzhen 518172, China. E-mail: tsunghui.chang@ieee.org.
}
\thanks{Manuscript received XX, 2021; revised XX, 2021.}}

\maketitle

\begin{abstract}
Non-negative matrix factorization (NMF) is a powerful tool for dimensionality reduction and clustering.
Unfortunately, the interpretation of the clustering results from NMF is difficult,
especially for the high-dimensional biological data without effective feature selection.
In this paper, we first introduce a row-sparse NMF with $\ell_{2,0}$-norm constraint (NMF\_$\ell_{20}$), where the basis matrix $\bm{W}$ is constrained by the $\ell_{2,0}$-norm, such that $\bm{W}$ has a row-sparsity pattern with feature selection.
It is a challenge to solve the model, because the $\ell_{2,0}$-norm is non-convex and non-smooth.
Fortunately, we prove that the $\ell_{2,0}$-norm satisfies the Kurdyka-\L{ojasiewicz} property.
Based on the finding, we present a proximal alternating linearized minimization algorithm and its monotone accelerated version to solve the NMF\_$\ell_{20}$ model.
In addition, we also present a orthogonal NMF with $\ell_{2,0}$-norm constraint (ONMF\_$\ell_{20}$) to enhance the clustering performance by using a non-negative orthogonal constraint.
We propose an efficient algorithm to solve ONMF\_$\ell_{20}$ by transforming it into a series of constrained and penalized matrix factorization problems. The results on numerical and scRNA-seq datasets demonstrate the efficiency of our methods in comparison with existing methods.
\end{abstract}

\IEEEpeerreviewmaketitle

\section{Introduction} 
\IEEEPARstart{W}{ith} the development of single cell RNA sequencing (scRNA-seq) technology, we can easily obtain biological profile data at single cell level from thousands of cells at the same time \cite{luecken2019current}. Clustering such scRNA-seq data has been becoming increasingly important for biological and medical applications \cite{kiselev2019challenges}.

Non-negative matrix factorization (NMF) and its variants have been widely used to solve some computational biological problems \cite{brunet2004metagenes,stravzar2016orthogonal,liu2017regularized,zhang2012discovery,chen2018discovery,zhang2019learning,fu2019nonnegative}.
Especially, they have achieved lots of successfully applications in scRNA-seq data clustering analysis \cite{shao2017robust,duren2018integrative,welch2019single}.
However, the NMF class algorithms expose some shortcomings when being applied to the clustering analysis of the high-dimensional biological data. 
It is well known that the high-dimensional biological data contains many noisy and redundant features which often impacts the performance of clustering algorithms.

To overcome the problems, sparse NMF methods have been proposed by adding sparseness constraints \cite{hoyer2004non,kim2007sparse,peharz2012sparse}.
At present, the proposed sparseness constraints, such as $\ell_1$ and $\ell_0$ norms, cannot identify real row-sparsity patterns of the basis matrix $\bm{W}$ in NMF, $\bm{X} \approx \bm{W}\bm{H}$ (see Figure \ref{alg-1}A and B), such that these sparse NMF methods cannot select the important features for clustering analysis.

Feature selection is a effective way that extract the informative features to improve the interpretability and performance of machine learning models \cite{li2017feature,gui2017feature}.
To enhance model interpretability, a common task is to search for some most important information features when NMF and its variants are used for high-dimensional data analysis.
Previously, $\ell_{2,1}$-norm constraint has been used in some supervised learning models to perform feature selection \cite{nie2010efficient,gui2017feature}. Moreover, $\ell_{2,0}$-norm is more desirable from the sparsity perspective, because it can select a specific number of the important information features \cite{Towards13,pang2019efficient,du2018exploiting}.

To integrate feature selection in the NMF model, we first present a row-sparse NMF with $\ell_{2,0}$-norm constraint (NMF\_$\ell_{20}$).
The basis matrix $\bm{W}$ is constrained by the $\ell_{2,0}$-norm, such that $\bm{W}$ has a row-sparsity pattern with feature selection (see Figure \ref{fig-1}C).
However, it is difficult to find an effective convergence algorithm to solve the NMF\_$\ell_{20}$ model because the $\ell_{2,0}$-norm constraint is non-convex and non-smooth.
Fortunately, we find that the $\ell_{2,0}$-norm satisfies the Kurdyka-\L{ojasiewicz} (K\L) property such that the traditional proximal gradient method can be used to solve a class of optimization problems with $\ell_{2,0}$-norm constraint.
For instance, the proximal alternating linearized minimization (PALM) algorithm has been proposed to solve a class of non-convex and non-smooth problems which satisfy K\L~property \cite{bo2014proximal}.
Based on the above point, we introduce the PALM algorithm and its variant, a monotone accelerated PALM (maPALM) algorithm, to solve the NMF\_$\ell_{2,0}$ model.
We prove that both PALM and maPALM algorithms converge to a critical point when they are used to solve the NMF\_$\ell_{20}$ model.

In addition, we also note that the orthogonal NMF (ONMF) which is a variant of NMF. 
It improves the clustering performance by adding the non-negative orthogonal constraint \cite{ding2006orthogonal,zhang2019greedy,wang2019clustering}.
Non-negative orthogonal matrix has the following two properties: (1) an orthonormal matrix forms a basis for a specific subspace, which facilitates geometric interpretation and signal reconstruction; (2) two non-negative vectors in the matrix are orthogonal if and only if their nonzero dimensions do not overlap. This may be the reason why ONMF is sometime more effective than NMF in clustering.

To integrate feature selection and non-negative orthogonal constraint in the NMF model, we also present a row-sparse ONMF with $\ell_{2,0}$-norm constraint (ONMF\_$\ell_{20}$).
We propose an efficient algorithm for ONMF\_$\ell_{20}$ model by using a penalty function method. The algorithm transforms ONMF\_$\ell_{20}$ into a series of subproblems so that the PALM and maPALM algorithms can be used to solve them. Our contributions of this paper are summarized as follows:
\begin{enumerate}
	\item We prove that the $\ell_{2,0}$-norm satisfies the K\L~property such that  a class of optimization problems with $\ell_{2,0}$-norm constraint can be solved by the PALM algorithm.
	\item An efficient algorithm (PALM) for NMF\_$\ell_{20}$, and its convergence property.
	\item An accelerated version of PALM (maPALM) for NMF\_$\ell_{20}$, and its convergence property.
	\item An efficient algorithm for ONMF\_$\ell_{20}$ by transforming it into a series of subproblems, and its convergence property.
	\item The application of our methods and the comparison with the competing methods using the simulated and scRNA-seq datasets. The results show that our methods are more effective in clustering accuracy and feature selection.
\end{enumerate}

\begin{figure}[h]
  \centering \includegraphics[width=1\linewidth]{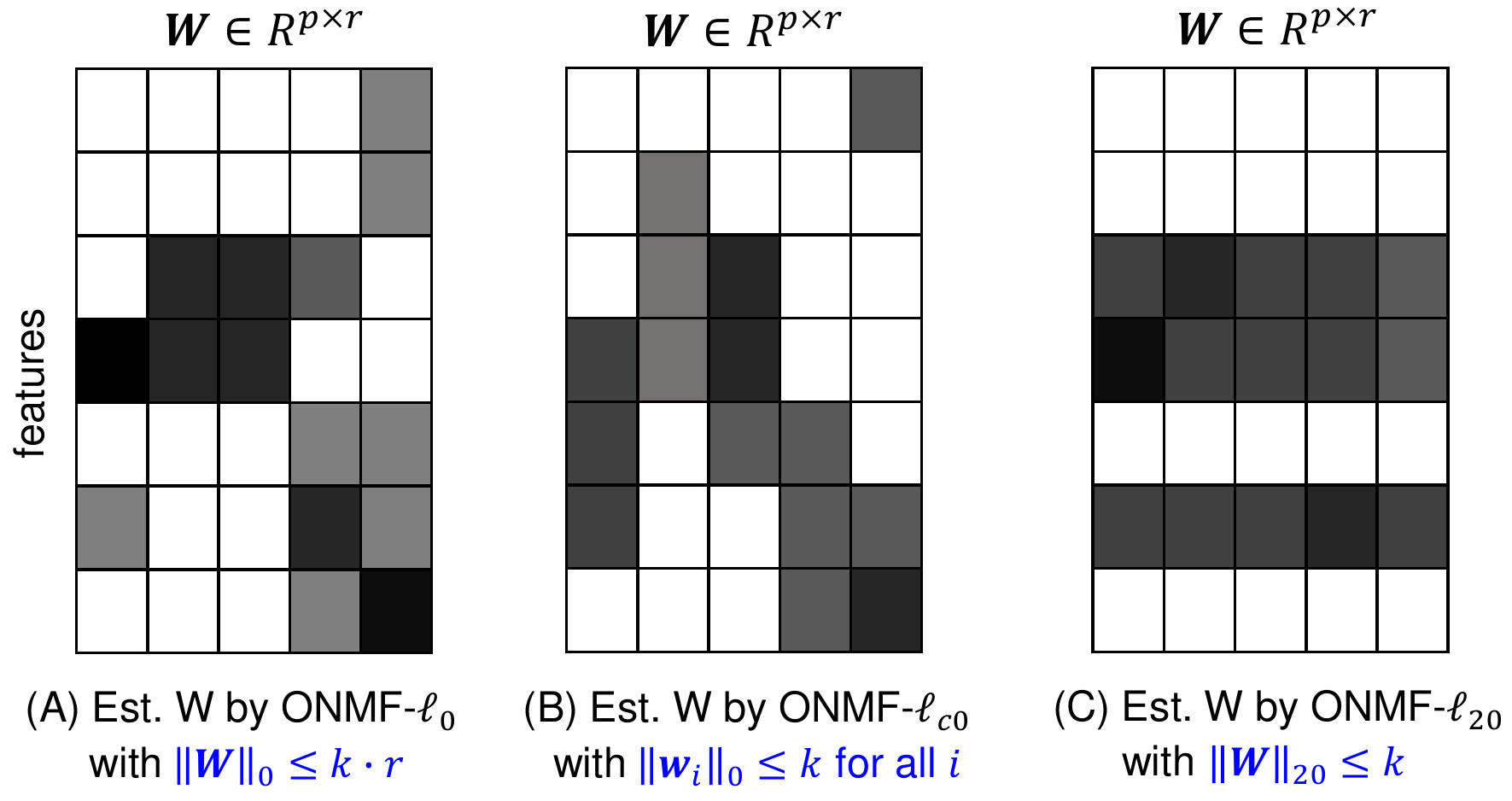}
  \caption{Illustration of these basis matrices $\bm{W}$ obtained from three structured sparse NMF models ($\bm{X} \approx \bm{W}\bm{H}$) with different constraints.
  (A) showing estimated $\bm{W}$ by NMF\_$\ell_0$,
  (B) showing estimated $\bm{W}$ by NMF\_$\ell_{20}$, and
  (C) showing estimated $\bm{W}$ by NMF\_$\ell_{c0}$, where dark boxes denote the non-zero coefficients and blank boxes represent zero coefficients. The constraint conditions of $\bm{W}$ in NMF\_$\ell_0$, NMF\_$\ell_{20}$ and NMF\_$\ell_{c0}$ are defined in Table \ref{tab-1}.
  }\label{fig-1}
\end{figure}

\begin{table}[ht]
\centering
\caption{Summary of notations.}
\begin{adjustbox}{width=1\columnwidth,center}\label{tab-1}
\begin{tabular}{l|l}
   \hline
   \textbf{Notation}  & \textbf{Meaning}\\
   \hline
   Normal font, e.g., x  & A scalar \\
   Bold lowercase, e.g., $\bm{w}$  & A vector\\
   Bold capital, e.g., $\bm{W}$    & A matrix\\
   \hline
   $\bm{X}$    & A $p$-by-$n$ matrix\\
   $\bm{W}$    & A $p$-by-$r$ matrix\\
   $\bm{H}$    & An $r$-by-$n$ matrix\\
   $\bm{I}_r$  & An $r \times r$ identity matrix\\
   \hline
   $\bm{h}^i$  &The $i$-th row of $\bm{H}$\\
   $\bm{h}_j$  &The $j$-th column of $\bm{H}$\\
   \hline
   $\|\cdot\|_1$ & $\ell_1$-norm for a vector\\
   $\|\cdot\|_0$ & $\ell_0$-norm for a vector\\
   $\|\cdot\|_2$ or $\|\cdot\|$ & $\ell_2$-norm for a vector\\
   \hline
   $\|\cdot\|_F$    & Frobenius norm for a matrix\\
   $\|\cdot\|_2$    & Spectral norm for a matrix\\
   $I()$            & Indicator function\\
   \hline
   \textbf{Model}   & \textbf{Constraint condition}\\
   \hline
   NMF\_$\ell_0$     & $\|\bm{W}\|_0 \leq k\cdot r$ (see Eq. \ref{equ-3})\\
   NMF\_$\ell_{20}$  & $\|\bm{W}\|_{2,0} \leq k$ (see Eq. \ref{equ-4})\\
   NMF\_$\ell_{c0}$  &  Column-wise sparsity of $\bm{W}$  (see Eq. \ref{equ-8})\\
   \hline
\end{tabular}
\end{adjustbox}
\end{table}

\section{Notations and definitions}
Given a matrix $\bm{W} \in \mathbb{R}^{p\times r}$, let $\bm{w}^i$ and $\bm{w}_j$ denote its $i$-th row and $j$-th column, respectively.
The Frobenius norm of $\bm{W} \in R^{p \times n}$ is defined as:
\begin{equation}\label{equ-1}
\|\bm{W}\|_F = \sqrt{\sum_{i=1}^p\sum_{j=1}^r {w_{ij}^2}} = \sqrt{tr(\bm{W}^T\bm{W})}.
\end{equation}
The spectral norm of $\bm{W}$ is the largest singular value of $\bm{W}$ and it is defined as:
\begin{equation}\label{equ-2}
\|\bm{W}\|_2 = \sigma_{\max}(\bm{W}),
\end{equation}
The $\ell_{0}$-norm of $\bm{W}$ is defined as:
\begin{equation}\label{equ-3}
\|\bm{W}\|_0 = \sum_{i=1}^p\sum_{j=1}^n{I(w_{ij} \neq 0)} = \sum_{i=1}^p \|\bm{w}^i\|_0 = \sum_{j=1}^r \|\bm{w}_j\|_0,
\end{equation}
where $I(x)=1$ if $x \neq 0$, $I(x)=0$ if $x = 0$.
The $\ell_{2,0}$-norm of $\bm{W}$ is defined as:
\begin{equation}\label{equ-4}
\|\bm{W}\|_{2,0} = \sum_{i=1}^p{I(\|\bm{w}^i\| \neq 0)} = \|(\|\bm{w}^1\|, \cdots, \|\bm{w}^p\|)\|_0,
\end{equation}
where $\| \cdot \|$ is the $\ell_2$-norm and $\| \bm{x} \| = \sum{x_i^2}$.
Briefly, $\|\bm{W}\|_{2,0}$ denotes the number of non-zero rows in $\bm{W}$.
More notations are summarized in Table \ref{tab-1} and the mathematical definitions for nonconvex optimization are summarized into the appendix \ref{appendix-a1}.
For simplicity, $\ell_{2,0}$ and $\|\bm{W}\|_{2,0}$ sometimes are abbreviated as $\ell_{20}$ and $\|\bm{W}\|_{20}$, respectively.

\section{Proposed framework}
\subsection{NMF and ONMF}
Given a data $\bm{X} \in \mathbb{R}^{p \times n}$ with $p$ features and $n$ samples, NMF model \cite{lee1999learning} can be written as follows:
\begin{equation}\label{equ-5}
\begin{aligned}
& \underset{\bm{W},\bm{H}} {\text{minimize}} && \|\bm{X} - \bm{W}\bm{H}\|_F^2\\
& \text{subject to}  && \bm{W} \in \mathbb{R}_+^{p \times r}, \bm{H} \in \mathbb{R}_+^{r \times n}.
\end{aligned}
\end{equation}
\cite{ding2006orthogonal} has reported that ONMF can improve the clustering performance by adding the non-negative orthogonal constraint on $\bm{H}$ . Thus, we introduce the following ONMF model:
\begin{equation}\label{equ-6}
\begin{aligned}
& \underset{\bm{W},\bm{H}} {\text{minimize}} && \|\bm{X} - \bm{W}\bm{H}\|_F^2\\
& \text{subject to}  && \bm{W} \in \mathbb{R}_+^{p \times r}, \bm{H} \in \mathbb{R}_+^{r \times n}, \bm{H}\bm{H}^T = \bm{I}_r.
\end{aligned}
\end{equation}
Theorem 1 in reference \cite{ding2006orthogonal} has shown that ONMF is equivalent to k-means clustering, because the nonnegative orthogonal matrix $\bm{H}$ has a good property (see remark \ref{remark-1}).
\begin{remark}\label{remark-1}
For the solution $\bm{H}$ of Eq. (\ref{equ-6}), it has at most one non-zero entry in each column, because $\bm{H}$ is non-negative and satisfies orthogonality $\bm{H}\bm{H}^T = \bm{I}_r$.
\end{remark}

\subsection{Structured sparse NMF (SSNMF)}
To integrate feature selection and non-negative orthogonal constraint in NMF model, we introduce a row-sparse ONMF with $\ell_{2,0}$-norm constraint (ONMF\_$\ell_{20}$):
\begin{equation}\label{equ-7}
\begin{aligned}
& \underset{\bm{W},\bm{H}} {\text{minimize}} && \|\bm{X} - \bm{W}\bm{H}\|_F^2\\
& \text{subject to}  && \bm{W}\in \Omega_w := \{\bm{W} \in \mathbb{R}_+^{p \times r}: \|\bm{W}\|_{2,0} \leq k\},\\
&                    && \bm{H}\in \Omega_h := \{\bm{H} \in \mathbb{R}_+^{r \times n}: \bm{H}\bm{H}^T = \bm{I}_r\},
\end{aligned}
\end{equation}
where $\|\bm{W}\|_{2,0}\leq k$ encourages $\bm{W}$ to be row sparse and select some most important features. We can also make $\bm{W}$ column sparse using the following constraint, named $\ell_{c0}$-norm, which uses $\ell_0$-norm to each column of $\bm{W}$:
\begin{equation}\label{equ-8}
	\Omega_w := \{\bm{W} \in \mathbb{R}_+^{p \times r}: \|\bm{w}_j\|_{0} \leq s, \forall j\}.
\end{equation}
The key to solve problem (\ref{equ-7}) is to remove its orthogonal constraint.
Based on the conclusion of Eq. (8) from reference \cite{wang2019clustering}, Eq. (\ref{equ-7}) is equivalent to:
\begin{equation}
\begin{aligned}\label{equ-9}
& \underset{\bm{W},\bm{H}} {\text{minimize}} && \|\bm{X} - \bm{W}\bm{H}\|_F^2\\
& \text{subject to}  && \bm{W}\in \Omega_w := \{\bm{W} \in \mathbb{R}_+^{p \times r}: \|\bm{W}\|_{2,0} \leq k\},\\
&                    && \bm{H}\in \{\bm{H} \in \mathbb{R}_+^{r \times n}: \|\bm{h}_j\|_1^2 = \|\bm{h}_j\|_2^2, \forall j\}.
\end{aligned}
\end{equation}
We consider its penalized formulation and present a SSNMF framework as follows:
\begin{equation}
\begin{aligned}\label{equ-10}
& \underset{\bm{W},\bm{H}} {\text{minimize}} && \frac{1}{2} \|\bm{X} - \bm{W}\bm{H}\|_F^2 + \frac{\rho}{2}\sum_{j=1}^n \Big((\bm{1}^T\bm{h}_j)^2 - \|\bm{h}_j\|_2^2\Big)\\
& \text{subject to}  && \bm{W}\in \Omega_w := \{\bm{W} \in \mathbb{R}_+^{p \times r}: \|\bm{W}\|_{2,0} \leq k\}, \\
&                    && \bm{H} \in  \Omega_h := \mathbb{R}_+^{r \times n}.
\end{aligned}
\end{equation}
Based on the SSNMF framework (\ref{equ-10}), we introduce the following four SSNMF models as follows:
\begin{itemize}
  \item \textbf{Row-sparse NMF} with $\ell_{2,0}$-norm constraint (NMF\_$\ell_{20}$). When $\rho=0$, the penalty term has no effect and Eq. (\ref{equ-10}) reduces to NMF\_$\ell_{20}$.
  \item \textbf{Row-sparse ONMF} with $\ell_{2,0}$-norm constraint (ONMF\_$\ell_{20}$). When $\rho$ is large enough, Eq. (\ref{equ-10}) reduces to ONMF\_$\ell_{20}$. As $\rho \rightarrow \infty$ in Eq. (\ref{equ-10}), the impact of the penalty grows, and the estimated $\bm{H}$ will approach a nonnegative orthogonal matrix.
  \item \textbf{Column-sparse NMF} with $\ell_{c,0}$-norm constraint (NMF\_$\ell_{c0}$). When $\rho=0$ and $\Omega_w$ is defined in Eq. (\ref{equ-8}),  Eq. (\ref{equ-10}) reduces to NMF\_$\ell_{c0}$.
  \item \textbf{Column-sparse ONMF} with $\ell_{c,0}$-norm constraint (ONMF\_$\ell_{c0}$). When $\rho$ is large enough and $\Omega_w$ is defined in Eq. (\ref{equ-8}), Eq. (\ref{equ-10}) reduces to ONMF\_$\ell_{c0}$.
\end{itemize}
We first present Proposition \ref{prop-1} to clarify the relationship between the solutions of Eq. (\ref{equ-10}) and Eq. (\ref{equ-7}), whose proof is shown in the appendix \ref{appendix-a2}.
Proposition \ref{prop-1} implies that we can solve Eq. (\ref{equ-7}) by repeatedly solving Eq. (\ref{equ-10}) with a gradually increasing $\rho$. Therefore, the key to solve the above four models (NMF\_$\ell_{20}$, NMF\_$\ell_{c0}$, ONMF\_$\ell_{20}$, and ONMF\_$\ell_{c0}$) is to solve the SSNMF framework (\ref{equ-10}).
\begin{prop}\label{prop-1}
Let $(\bm{W}^*, \bm{H}^*)$  be a local minimizer of Eq. (\ref{equ-10}). Then for any $\rho > 0$, $(\bm{W}^*, \bm{H}^*)$ is also the feasible and local minimizer of Eq. (\ref{equ-7}).
\end{prop}

Alternating minimization is a popular strategy to solve the constrained and penalized matrix factorization problem in Eq. (\ref{equ-10}). Recently, a PALM algorithm has been proposed to solve a class of constrained and penalized matrix factorization problems which satisfies K\L~property \cite{bo2014proximal}. We present Theorem \ref{lem-1} to show that Eq. (\ref{equ-10}) satisfies K\L~property, such that the PALM algorithm can be used to solve it.

\begin{theorem}\label{lem-1}
	$\Phi = F(\bm{W},\bm{H}) + \delta_{\bm{W}\geq 0} + \delta_{\bm{H}\geq 0} + \delta_{\|\bm{W}\|_{2,0} \leq k}$ is a semi-algebraic function and it satisfies the K\L~property, where $F = \frac{1}{2} \|\bm{X} - \bm{W}\bm{H}\|_F^2 + \frac{\rho}{2}\sum_{j=1}^n \Big((\bm{1}^T\bm{h}_j)^2 - \|\bm{h}_j\|_2^2\Big)$, and $\delta_{\|\bm{W}\|_{2,0} \leq k}$ is zero if $\|\bm{W}\|_{2,0} \leq k$, otherwise $+ \infty$.
\end{theorem}
\begin{proof}
	Remark 8 in \cite{bo2014proximal} shows the following conclusions:
	(1) $\|\cdot\|_0$ and $\|\cdot\|_2$ are semi-algebraic functions;
	(2) The indicator function in a semi-algebraic set is semi-algebraic;
	(3) Any composition of semi-algebraic function remains to be semi-algebraic;
	(4) The finite sums of semi-algebraic functions remain semi-algebraic.
	Based the above conclusions, we know that
	(1) $F$, $\delta_{\bm{W}\geq 0}$ and $\delta_{\bm{H}\geq 0}$ are semi-algebraic functions, respectively;
	(2) $\delta_{\|\bm{W}\|_{2,0} \leq k}$ is a semi-algebraic function because $\|\bm{W}\|_{2,0}:=\|(\|\bm{w}^1\|_2, \cdots, \|\bm{w}^p\|_2)\|_0$ is a composition of semi-algebraic function $\|\cdot\|_0$ and $\|\cdot\|_2$.	
	Thus, we prove that $\Phi$ is a semi-algebraic function because it is a sum of four semi-algebraic functions. In addition, we observe that $\Phi$ is a proper and lower semicontinuous function.
	Based on the Theorem 3 in \cite{bo2014proximal}, a proper, lower semicontinuous and semi-algebraic function $\Phi$ satisfies the K\L~property. The relevant mathematical definitions are shown in Appendix \ref{appendix-a1}.	
\end{proof}

By  the way, the proof of Theorem \ref{lem-1} also implies that the $\ell_{2,0}$-norm satisfies the K\L~property such that the PALM can be used to solve a class of optimization problems with $\ell_{2,0}$-norm constraint.


\section{Optimization method}
\subsection{PALM and maPALM}
We first introduce a general constrained and penalized matrix factorization model as follows:
\begin{equation} 
	\begin{aligned}\label{equ:11}
		& \underset{\bm{W},\bm{H}} {\text{minimize}} && F(\bm{W},\bm{H}) = \frac{1}{2} \|\bm{X} - \bm{W}\bm{H}\|_F^2 + \phi(\bm{W}) + \varphi(\bm{H})\\
		& \text{subject to}  && \bm{W}  \in \Omega_w , \bm{H} \in  \Omega_h.
	\end{aligned}
\end{equation}
Obviously, the SSNMF framework in Eq. (\ref{equ-10}) is a special case of Eq. (\ref{equ:11}).
To solve Eq. (\ref{equ:11}) using the PALM algorithm, we need to perform a projected gradient descent step with respect to $\bm{H}$ and $\bm{W}$ for $t =1,2,\cdots$:
\begin{subequations}\label{equ-11}
\begin{align*}
\bm{H}^{t+1} & = \mathcal{P}_{H\in \Omega_H} \Big\{ \bm{H}^{t} - \frac{1}{d_H^t} \nabla_H F(\bm{W}^{t}, \bm{H}^{t}) \Big\}, \\
\bm{W}^{t+1} & = \mathcal{P}_{W\in \Omega_W} \Big\{ \bm{W}^{t} - \frac{1}{d_W^t} \nabla_W F(\bm{W}^{t}, \bm{H}^{t+1}) \Big\},
\end{align*}
\end{subequations}
where $d_H^t$ and $d_W^t$ are two step-size parameters, $\mathcal{P}_{H\in \Omega_H}\{\cdot\}$ and $\mathcal{P}_{W\in \Omega_W}\{\cdot\}$ are two projection operations onto $\Omega_H$ and $\Omega_W$, respectively. A basic algorithmic framework for solving Eq. (\ref{equ:11}) is shown in Algorithm \ref{alg-1}.

Because of the linearization of PALM, its convergence speed may be slow. Accelerated proximal gradient  method has been widely used to solve convex optimization problems. Unfortunately, if the non-monotone accelerated proximal gradient uses a bad extrapolation for some non-convex problems, then it may not converge to a critical point \cite{li2015accelerated}. Fortunately, the monotone accelerated method guarantees convergence for a non-convex problem by ensuring that its objective function value decreases every iteration \cite{li2017convergence,li2015accelerated}.

\begin{algorithm}[htbp]
	\caption{PALM for solving Eq. (\ref{equ:11}).} \label{alg-1}
	\begin{algorithmic}[1]
		\REQUIRE $\bm{X}\in \mathbb{R}^{p\times n}$ and $\epsilon>0$.
		\ENSURE $\bm{W}\in \mathbb{R}^{p\times r}$ and $\bm{H}\in \mathbb{R}^{r\times n}$.
		\STATE Initialize $(\bm{W}^{0},\bm{H}^{0})$ and set $t=0$
		\REPEAT
		\STATE Compute $d_H^t$
		\STATE $\bm{H}^{t+1}  = \mathcal{P}_{H\in \Omega_H} \Big\{ \bm{H}^{t} - \frac{1}{d_H^t} \nabla_H F(\bm{W}^{t}, \bm{H}^{t}) \Big\}$
		\STATE Compute $d_W^t$
		\STATE $\bm{W}^{t+1}  = \mathcal{P}_{W\in \Omega_W} \Big\{ \bm{W}^{t} - \frac{1}{d_W^t} \nabla_W F(\bm{W}^{t}, \bm{H}^{t+1}) \Big\}$
		\STATE $t = t + 1$
		\UNTIL $\frac{\|(\bm{W}^{t},\bm{H}^{t})-(\bm{W}^{t-1},\bm{H}^{t-1})\|}{\| (\bm{W}^{t-1},\bm{H}^{t-1}) \|} < \epsilon$
		\RETURN $\bm{W}:=\bm{W}^{t}$ and $\bm{H}:=\bm{H}^{t}$.
	\end{algorithmic}
\end{algorithm}

\begin{algorithm}[htbp]
\caption{maPALM for solving Eq. (\ref{equ:11}).} \label{alg-2}
\begin{algorithmic}[1]
\REQUIRE $\bm{X}\in \mathbb{R}^{p\times n}$, $\epsilon>0$ and $\omega_0$.
\ENSURE $\bm{W}\in \mathbb{R}^{p\times r}$ and $\bm{H}\in \mathbb{R}^{r\times n}$.
\STATE Initialize $(\bm{W}^{-1},\bm{H}^{-1})= (\bm{W}^{0},\bm{H}^{0})$ and $t=0$
\REPEAT
\STATE $\widetilde{\bm{H}}^t = \bm{H}^{t} + \omega_{t} (\bm{H}^{t}-\bm{H}^{t-1})$
\STATE Compute $d_H^t$
\STATE $\widetilde{\bm{H}}^{t+1} = \mathcal{P}_{H\in \Omega_H} \Big\{ \widetilde{\bm{H}}^{t} - \frac{1}{d_H^t} \nabla_H F(\bm{W}^{t}, \widetilde{\bm{H}}^{t}) \Big\}$
\STATE $\widetilde{\bm{W}}^t = \bm{W}^{t} + \omega_{t} (\bm{W}^{t}-\bm{W}^{t-1})$
\STATE Compute $d_W^t$
\STATE $\widetilde{\bm{W}}^{t+1} = \mathcal{P}_{W\in \Omega_W} \Big\{ \widetilde{\bm{W}}^{t} - \frac{1}{d_W^t} \nabla_W F(\widetilde{\bm{W}}^{t}, \widetilde{\bm{H}}^{t+1}) \Big\}$
\IF{$F(\widetilde{\bm{W}}^{t+1}, \widetilde{\bm{H}}^{t+1}) \leq F(\bm{W}^{t},\bm{H}^{t})$}
\STATE $(\bm{W}^{t+1},\bm{H}^{t+1}) := (\widetilde{\bm{W}}^{t+1},\widetilde{\bm{H}}^{t+1})$
\ELSE
\STATE update $(\bm{W}^{t+1},\bm{H}^{t+1})$ using Eq. (\ref{equ-11})
\ENDIF
\STATE Compute $\omega_{t}$ and set $t = t + 1$
\UNTIL $\frac{\|(\bm{W}^{t},\bm{H}^{t})-(\bm{W}^{t-1},\bm{H}^{t-1}) \|}{\| (\bm{W}^{t-1},\bm{H}^{t-1}) \|} < \epsilon$
\RETURN $\bm{W} := \bm{W}^{t}$ and $\bm{H} := \bm{H}^{t}$.
\end{algorithmic}
\end{algorithm}

To this end, we develop a monotone accelerated PALM framework and its details is shown in Algorithm \ref{alg-2}, which can be regarded as a special case of the block prox-linear method \cite{xu2017globally}. Similarly, maPALM ensures the objective function value decreases. Reference \cite{li2015accelerated,xu2017globally}, $\omega_k$ ($k=0, 1,2,\cdots$) in Algorithm \ref{alg-2} is dynamically updated by
\begin{equation}\label{equ-12}
  \omega_{k} = \frac{\tau_k - 1}{\tau_{k + 1}},
\end{equation}
where $\tau_{0}  = 1$ and $\tau_{k+1} = \Big(1+\sqrt{1+4\tau_k^2} \Big)/2$.

We note that maPALM reduces to PALM when $\omega_{t}  = 0$ for all $t$.
Especially, NMF\_$\ell_{20}$ and NMF\_$\ell_{c0}$ can be solved by Algorithm \ref{alg-1} and \ref{alg-2} with $\phi(\bm{W})= \varphi(\bm{H})=0$. 
In addition, the conclusion of Proposition \ref{prop-1} implies that we can solve Eq. (\ref{equ-7}) by repeatedly solving Eq. (\ref{equ-10}) with a gradually increasing $\rho$. To this end, we can develop an efficient algorithm for ONMF\_$\ell_{20}$ or ONMF\_$\ell_{c0}$ by turning ONMF\_$\ell_{20}$ or ONMF\_$\ell_{c0}$ into a sequence subproblems.

In summary, the key to solving NMF\_$\ell_{20}$, NMF\_$\ell_{c0}$,  ONMF\_$\ell_{20}$ and ONMF\_$\ell_{c0}$ is to solve the SSNMF framework in Eq. (\ref{equ-10}). 
Below we show the details of using the maPALM algorithm to solve it.

\subsection{Solve SSNMF with $\ell_{2,0}$-norm constraint}
We use maPALM to solve a row-sparse SSNMF in Eq. (\ref{equ-10}) with $\ell_{2,0}$-norm constraint:
\begin{equation}\label{equ-13}
\begin{aligned}
& \underset{\bm{W},\bm{H}} {\text{minimize}} && \frac{1}{2} \|\bm{X} - \bm{W}\bm{H}\|_F^2 + \frac{\rho}{2}\sum_{j=1}^n \Big((\bm{1}^T\bm{h}_j)^2 - \|\bm{h}_j\|_2^2\Big) \\
& \text{subject to}  && \bm{W} \in \mathbb{R}_+^{p \times r}, \|\bm{W}\|_{2,0} \leq k, \bm{H} \in \mathbb{R}_+^{r \times n}.
\end{aligned}
\end{equation}
Let $F = \frac{1}{2} \|\bm{X} - \bm{W}\bm{H}\|_F^2 + \frac{\rho}{2}\sum_{j=1}^n \Big((\bm{1}^T\bm{h}_j)^2 - \|\bm{h}_j\|_2^2\Big)$, then
\begin{subequations}\label{equ-14}
\begin{gather}
\nabla_H F = \bm{W}^T\bm{W}\bm{H} - \bm{W}^T\bm{X} + \rho\bm{1}_{r\times r}\bm{H} - \rho\bm{H},\\
\nabla_W F = \bm{W}\bm{H}\bm{H}^T - \bm{X}\bm{H}^T.
\end{gather}
\end{subequations}
And the Hessian matrices of $F$ with respect to $\bm{W}$ and $\bm{H}$ are
\begin{subequations} \label{equ-15}
\begin{gather}
\nabla_W^2 F = (\bm{H}\bm{H}^T) \otimes \bm{I}_p,  \\
\nabla_H^2 F = \bm{I}_p \otimes (\bm{W}^T\bm{W} + \rho\bm{1}_{r\times r} - \rho \bm{I}_r ).
\end{gather}
\end{subequations}
where $\otimes$ is Kronecker product and $\bm{I}_p \in \mathbb{R}^{p \times p}$ is an identity matrix.
To use maPALM to solve Eq. (\ref{equ-13}), we need to calculate Lipschitz constant to determine the step size.
The Lemma 2 in reference \cite{guan2012nenmf} shows that $\nabla_W F$ and $\nabla_H F$ are Lipschitz continuous, the Lipschitz constant of $\nabla_W F$ is the largest singular value of $\nabla_W^2F$, \emph{i.e.}, $L_W = \|\bm{H}\bm{H}^T\|_2$, and the Lipschitz constant of $\nabla_H F$ is the largest singular value of $\nabla_W^2F$,
\emph{i.e.}, $L_H = \|\bm{W}^T\bm{W} + \rho \bm{1}_{r\times r} - \rho \bm{I}_r \|_2$.
Thus, we can set $d_W^t=L_W$ and $d_H^t=L_H$ in Algorithms \ref{alg-1} to \ref{alg-2}.

\textbf{1) Optimize W}.
Specifically, to obtain the update of $\bm{W}$ for Eq. (\ref{equ-13}), we need to solve a proximal map $\bm{W}^{t+1} := \mathcal{P}_{W\in \Omega_W}\{\overline{\bm{W}}\}$ as follows:
\begin{equation}\label{equ-16}
\begin{aligned}
& \underset{\bm{W}} {\text{minimize}} && \|\bm{W} - \overline{\bm{W}}\|_F^2 \\
& \text{subject to}  && \bm{W} \in \mathbb{R}_+^{p \times r}, \|\bm{W}\|_{2,0} \leq k,
\end{aligned}
\end{equation}
where $\overline{\bm{W}} := \bm{W}^{t} - \frac{1}{d_W^t} \nabla_W F(\bm{W}^{t}, \bm{H}^{t})$.
We propose Proposition \ref{prop-2} to solve the proximal map.
To this end, we introduce the following mathematical definitions.
\begin{defn}\label{defn-1}
$\mathrm{SupportNorm}(\bm{W}, k)$ is a set of indices of $\bm{z}$ with the largest $k$ values where $\bm{z} = (\|\bm{w}^1\|, \cdots, \|\bm{w}^p\|)$ and $\bm{w}^i$ denotes $i$-th row of $\bm{W}$.
\end{defn}
\begin{defn}\label{defn-2}
For a given matrix $\bm{W} \in \mathbb{R}^{p\times r}$, $\mathrm{RS}_k(\bm{W})$ is also a $p\times r$ matrix which is defined as:
\begin{equation} \label{equ-17}
 [\mathrm{RS}_k(\bm{W})]_{ij} =
\begin{cases}
    W_{ij},  & \text{if }i \in \mathrm{SupportNorm}(\bm{W}, k),\\
    0,       & \text{otherwise},
\end{cases}
\end{equation}
where $\mathrm{SupportNorm}(\bm{W}, k)$ is defined in Definition \ref{defn-1}. $\mathrm{RS}_k(\bm{W})$ only keeps $k$ non-zero rows with the largest $\ell_2$-norm values in $\bm{W}$.
\end{defn}
\begin{defn}\label{defn-3}
For a given matrix $\bm{W} \in \mathbb{R}^{p\times r}$, $P_+(\bm{W})$ is defined as follows:
\begin{equation} \label{equ-18}
 [P_+(\bm{W})]_{ij} =
\begin{cases}
    W_{ij},  & \text{if}~W_{ij}>0~\forall i~\text{and}~j,\\
    0,       & \text{otherwise}.
\end{cases}
\end{equation}
\end{defn}
\begin{prop}\label{prop-2}
(Proximal map formula of $\bm{W}$ for ONMF\_$\ell_{20}$)
Let $\overline{\bm{W}} \in \mathbb{R}^{p\times r}$, then Eq. (\ref{equ-16}) has a closed-form solution
\begin{equation}\label{equ-19}
 \mathcal{P}_{W\in \Omega_W}\{\overline{\bm{W}}\} := \mathrm{RS}_k(P_+(\overline{\bm{W}})),
\end{equation}
where $\mathrm{RS}_k(\cdot)$ and $P_+(\cdot)$ are defined in Definition \ref{defn-2} and \ref{defn-3}, respectively.
\end{prop}
\begin{proof}
Suppose that the optimal solution of Eq. (\ref{equ-16}) is $\widehat{\bm{W}}$, then we can easily observe that $\widehat{W}_{ij}$ must be zero when $\overline{W}_{ij} < 0$. So, Eq. (\ref{equ-16}) is equivalent to
\begin{equation*}
\underset{\bm{W}} {\text{minimize}}~\|\bm{W} - P_+(\overline{\bm{W}})\|_F^2,~\text{subject to}~\|\bm{W}\|_{2,0} \leq k.
\end{equation*}
Due to the constraint $\|\bm{W}\|_{2,0} \leq k$, the optimal solution only keeps up to $k$ non-zero rows.
So, Eq. (\ref{equ-16}) has a closed-form solution $\mathrm{RS}_k(P_+(\overline{\bm{W}}))$.
\end{proof}

\textbf{2) Optimize H}.
To obtain the update of $\bm{H}$ for Eq. (\ref{equ-13}) , we need to solve a proximal map $\bm{H}^{t+1} := \mathcal{P}_{H\in \Omega_H}\{\overline{\bm{H}}\}$ as follows:
\begin{equation}\label{equ-20}
\underset{\bm{H}} {\text{minimize}}~\|\bm{H} - \overline{\bm{H}}\|_F^2,~\text{subject to}~\bm{H} \in \mathbb{R}_+^{r \times n},
\end{equation}
where $\overline{\bm{H}} := \bm{H}^{t} - \frac{1}{d_H^t} \nabla_H F(\bm{W}^{t+1}, \bm{H}^{t})$.
We propose the following Proposition \ref{prop-3} to solve the above problem.

\begin{prop} \label{prop-3}
(Proximal map formula of $\bm{H}$ for ONMF\_$\ell_{20}$) Let $\overline{\bm{H}} \in \mathbb{R}^{r\times n}$, then  (\ref{equ-20}) has a closed-form solution:
\begin{equation}\label{equ-21}
 \mathcal{P}_{H\in \Omega_H}\{\overline{\bm{H}}\} := P_+(\overline{\bm{H}}),
\end{equation}
where $P_+(\cdot)$ is defined in Definition \ref{defn-3}.
\end{prop}

\textbf{3) Algorithm for solving Eq. (\ref{equ-13})}.
Based on the Propositions \ref{prop-2} and \ref{prop-3}, we develop a maPALM algorithm to solve Eq. (\ref{equ-13}) and the detailed algorithm is given in Algorithm \ref{alg-3}.
To maintain monotonicity, maPALM needs to use a suitable $\omega_{t}$ by checking the objective function value. If the objective function value becomes larger, then we obtain the update rule based on the traditional projected gradient descent method.
Monotonicity can ensure that Algorithm \ref{alg-3} converges to a critical point for any initial point.
The following Theorem \ref{theorem-1} gives the details on the convergence analysis of Algorithm \ref{alg-3}.
\begin{algorithm}[htbp]
\caption{maPALM for solving Eq. (\ref{equ-13}).} \label{alg-3}
\begin{algorithmic}[1]
\REQUIRE $\bm{X}\in \mathbb{R}^{p\times n}$, $\rho\geq0$, $\epsilon>0$, $k$ (nonzero rows).
\ENSURE  $\bm{W}\in \mathbb{R}^{p\times r}$ and $\bm{H}\in \mathbb{R}^{r\times n}$.
\STATE Initialize $(\bm{W}^{-1},\bm{H}^{-1})= (\bm{W}^{0},\bm{H}^{0})$, $\tau_0 = 1$ and $t=0$
\REPEAT
\STATE Compute $\omega_{t}  = \frac{\tau_t - 1}{\tau_{t + 1}}$ where $\tau_{t+1}  = \frac{1+\sqrt{1+4\tau_t^2}}{2}$
\STATE $\widetilde{\bm{H}}^t = \bm{H}^{t} + \omega_{t} (\bm{H}^{t}-\bm{H}^{t-1})$
\STATE $d_H^t = \|(\bm{W}^{t})^T(\bm{W}^{t}) + \rho \bm{1}_{r\times r} - \rho \bm{I}_r \|_2$
\STATE $\bm{H}^{t+1} = P_+ \big( \widetilde{\bm{H}}^{t} - \frac{1}{d_H^t} \nabla_H F(\bm{W}^{t}, \widetilde{\bm{H}}^{t}) \big)$
\STATE $\widetilde{\bm{W}}^t = \bm{W}^{t} + \omega_{t} (\bm{W}^{t}-\bm{W}^{t-1})$
\STATE $d_W^t = \|(\bm{H}^{t+1})(\bm{H}^{t+1})^T\|_2$
\STATE $\bm{W}^{t+1} = \mathrm{RS}_k \big( \widetilde{\bm{W}}^{t} - \frac{1}{d_W^t} \nabla_W F(\widetilde{\bm{W}}^{t}, \bm{H}^{t+1}) \big)$

\IF{$F(\widetilde{\bm{W}}^{t+1}, \widetilde{\bm{H}}^{t+1}) \leq F(\bm{W}^{t},\bm{H}^{t})$}
\STATE $(\bm{W}^{t+1},\bm{H}^{t+1}) := (\widetilde{\bm{W}}^{t+1},\widetilde{\bm{H}}^{t+1})$
\ELSE
\STATE $d_H^t = \|(\bm{W}^{t})^T(\bm{W}^{t}) + \rho \bm{1}_{r\times r} - \rho \bm{I}_r \|_2$
\STATE $\bm{H}^{t+1} = P_+ \big( \bm{H}^{t} - \frac{1}{d_H^t} \nabla_H F(\bm{W}^{t}, \bm{H}^{t}) \big)$
\STATE $d_W^t = \|(\bm{H}^{t+1})(\bm{H}^{t+1})^T\|_2$
\STATE $\bm{W}^{t+1} = \mathrm{RS}_k \big( \bm{W}^{t} - \frac{1}{d_W^t} \nabla_W F(\bm{W}^{t}, \bm{H}^{t+1}) \big)$
\ENDIF
\STATE $t = t + 1$
\UNTIL $\frac{\|(\bm{W}^{t},\bm{H}^{t})-(\bm{W}^{t-1},\bm{H}^{t-1}) \|}{\| (\bm{W}^{t-1},\bm{H}^{t-1}) \|} < \epsilon$
\RETURN $\bm{W} := \bm{W}^{t}$ and $\bm{H} := \bm{H}^{t}$.
\end{algorithmic}
\end{algorithm}

\textbf{Algorithm for NMF\_$\ell_{20}$}. Algorithm \ref{alg-3} with $\rho=0$ can effectively solve NMF\_$\ell_{20}$.

\textbf{Initialization}. We can adopt two ways to generate the initial point of Algorithm \ref{alg-3}. One is to use the random vectors from a standard normal distribution to initialize $\bm{W}$ and $\bm{H}$. The other is to use the solution $\bm{W}$ and $\bm{H}$ derived by the traditional NMF to initialize them. The second way is a good guess. So, if not specified, we use the second way to initialize by default.

\textbf{Step-size}. We can use a fixed value of the step-size $1/d_H^t$ and $1/d_W^t$ in Algorithm \ref{alg-3}, and also try to perform an approximate backtracking line search from $t \in (0, 0.5)$. Specifically, we use the fixed $d_W^t=L_W$ and $d_H^t=L_H$ in this paper.

\textbf{4) Computation cost}.
The computational complexity of Algorithm \ref{alg-3} depends on the number of iterations.
At each iteration, only two simple closed-form solutions need to be computed with respect to $\bm{H}$ and $\bm{W}$ in the steps 6 and 9, respectively.
For each update of $\bm{H}$, the most costly step is the calculation of $\nabla_H F$, which requires a computation cost of $O(pr^2 + nr^2)$.
For each update of $\bm{W}$, the most costly steps is the calculation of $\nabla_W F$, which requires a computation cost of $O(pr^2 + nr^2 + pnr)$.
In addition, the calculation of objective function value requires a computation cost of $O(pnr + pn +rn)$.
Thus, each iteration of Algorithm \ref{alg-3} requires a computation cost of $O(pr^2 + nr^2 + pnr)$.

\textbf{5) Convergence analysis}.
Based on some results from references \cite{xu2017globally,bo2014proximal}, we propose the following Theorem \ref{theorem-1} to show that Algorithm \ref{alg-3} globally converges to a critical point.

\begin{theorem}\label{theorem-1}
(Global convergence of Algorithm \ref{alg-3})
Let $\{(\bm{W}^{(i)}, \bm{H}^{(i)})\}_{i=1}^{\infty}$ be a sequence generated from any starting point $(\bm{W}^{(0)}, \bm{H}^{(0)})$ by Algorithm \ref{alg-3}. If $\{(\bm{W}^{(i)}, \bm{H}^{(i)})\}$ are bounded, then the objective is non-increasing and the sequence has finite length and converges to a critical point.
\end{theorem}
\begin{proof}
Based on Theorem \ref{lem-1}, the objective function $\Phi$ is a semi-algebraic function and satisfies the KL property. Checking the assumptions of Theorem 2 in reference \cite{xu2017globally}, we observe that all  assumptions required in Algorithm \ref{alg-3} are clearly satisfied.
So, we have Theorem \ref{theorem-1}.
\end{proof}

\subsection{Solve SSNMF with $\ell_{c,0}$-norm constraint}
Let $\Omega_w := \{\bm{W} \in \mathbb{R}_+^{p \times r}: \|\bm{w}_j\|_{0} \leq k, \forall j\}$ in Eq. (\ref{equ-10}), we consider a column-wise SSNMF with $\ell_{c,0}$-norm constraint by using $\ell_{0}$-norm for each column of $\bm{W}$ as follows:
\begin{equation} \label{equ-22}
\begin{aligned}
& \underset{\bm{W},\bm{H}} {\text{minimize}} && \frac{1}{2} \|\bm{X} - \bm{W}\bm{H}\|_F^2 + \frac{\rho}{2}\sum_{j=1}^n \Big((\bm{1}^T\bm{h}_j)^2 - \|\bm{h}_j\|_2^2\Big) \\
& \text{subject to}  && \bm{W} \in \mathbb{R}_+^{p \times r}, \|\bm{w}_j\|_{0} \leq k, \forall j, \bm{H} \in \mathbb{R}_+^{r \times n}.
\end{aligned}
\end{equation}
To obtain the update of $\bm{W}$, we need to solve a proximal map $\bm{W}^{t+1} := \mathcal{P}_{W\in \Omega_W}\{\overline{\bm{W}}\}$:
\begin{equation} \label{equ-23}
\begin{aligned}
& \underset{\bm{W}} {\text{minimize}} && \|\bm{W} - \overline{\bm{W}}\|_F^2 \\
& \text{subject to}  && \bm{W} \in \mathbb{R}_+^{p \times r},  \|\bm{w}_j\|_{0} \leq k, \forall j.
\end{aligned}
\end{equation}
\begin{defn}\label{defn-4}
For a given matrix $\bm{W} \in \mathbb{R}^{p\times r}$, $\mathrm{CS}_k(\bm{W}) \in \mathbb{R}^{p\times r}$ ($k\leq p$) is defined as:
\begin{equation} \label{equ-24}
 [\mathrm{CS}_k(\bm{W})]_{ij} =
\begin{cases}
    W_{ij},  & j \in \mathrm{SupportNorm}(\bm{w}_j, k),\\
    0,       & \text{otherwise},
\end{cases}
\end{equation}
where $\bm{w}_j$ denotes $j$-th column of $\bm{W}$ and $\mathrm{SupportNorm}(\cdot, k)$ is defined in Definition \ref{defn-1}.
\end{defn}
\begin{prop} \label{prop-4}
(Proximal map formula  of $\bm{W}$ for ONMF\_$\ell_{c0}$) Let $\overline{\bm{W}} \in \mathbb{R}^{p\times r}$, then Eq. (\ref{equ-23}) has a closed-form solution:
\begin{equation}\label{equ-25}
  \mathcal{P}_{W\in \Omega_W}\{\overline{\bm{W}}\} := \mathrm{CS}_k(P_+(\bm{W})),
\end{equation}
where $\mathrm{CS}_k(\cdot)$ and $P_+(\cdot)$ are defined in Definition \ref{defn-4} and \ref{defn-3}, respectively.
\end{prop}
\begin{proof}
Suppose that the optimal solution of Eq. (\ref{equ-23}) is $\widehat{\bm{W}}$, then we can observe that $\widehat{W}_{ij}$ must be zero when $\overline{W}_{ij} < 0$.
So, Eq. (\ref{equ-23}) is equivalent to
\begin{equation*}
\underset{\bm{W}} {\text{minimize}}~\|\bm{W} - P_+(\overline{\bm{W}})\|_F^2~\text{subject to}~\|\bm{w}_j\|_{0} \leq k, \forall j.
\end{equation*}
Due to the constraint $\|\bm{w}_j\|_{0} \leq k$, the optimal solution only keeps $k$ non-zero elements with the largest absolute values for each column of $\bm{W}$. So, Eq. (\ref{equ-23}) has a closed-form solution $\mathrm{CS}_k(P_+(\overline{\bm{W}}))$.
\end{proof}

\textbf{Algorithm for solving Eq. (\ref{equ-22})}.
Based on Proposition \ref{prop-4}, we can solve Eq. (\ref{equ-22}) by replacing $\mathrm{CS}_k(\cdot)$ with $\mathrm{RS}_k(\cdot)$ in Algorithm \ref{alg-3}.
Similar to Theorem \ref{theorem-1}, the variant of Algorithm \ref{alg-3} has theoretical convergence guarantees.

\subsection{Solve ONMF\_$\ell_{20}$ and ONMF\_$\ell_{c0}$}
We develop an efficient algorithm for ONMF\_$\ell_{20}$ by turning it into 
a series of constrained and penalized matrix factorization problems, i.e., 
Eq. (\ref{equ-13}) with different $\rho$,  which can be solved by the PALM or maPALM algorithm.
The detailed algorithm is given in Algorithm \ref{alg-4}.

\begin{algorithm}[htbp]
\caption{ONMF\_$\ell_{20}$ Algorithm} \label{alg-4}
\begin{algorithmic}[1]
\REQUIRE $\bm{X}\in \mathbb{R}^{p\times n}$, $\rho\geq0$ and $\gamma>1$.
\ENSURE $\bm{W}\in \mathbb{R}^{p\times r}$ and $\bm{H}\in \mathbb{R}^{r\times n}$.
\STATE Initialize $(\bm{W}^{0},\bm{H}^{0})$
\FOR {\texttt{$k=1,2,\cdots,K$}}
\STATE Obtain $(\bm{W},\bm{H})$ by solving Eq. (\ref{equ-14}) with $\rho$ using PALM or maPALM with the initialization $(\bm{W}^{0},\bm{H}^{0})$.
\STATE $\rho := \gamma \rho$ and $(\bm{W}^{0},\bm{H}^{0}) := (\bm{W},\bm{H})$
\ENDFOR
\RETURN $\bm{W}$ and $\bm{H}$.
\end{algorithmic}
\end{algorithm}

In Algorithm \ref{alg-4}, we set the default parameters $\rho=0.1$, $\gamma=1.5$, and $K=10$. As $\rho \rightarrow \infty$ in the step 4 of Algorithm \ref{alg-4}, the impact of the penalty grows, and the estimated $\bm{H}$ will approach a non-negative orthogonal matrix.
Finally, we can also use a method similar to Algorithm \ref{alg-4} to solve ONMF\_$\ell_{c0}$ by turning it into a series of constrained and penalized matrix factorization problems, i.e., Eq. (\ref{equ-22}) with different $\rho$.
\section{Experiments}
We evaluate the effectiveness of these proposed SSNMF methods for clustering task and compare them with state-of-the-art matrix factorization, k-means and two sparse k-means methods on the synthetic and scRNA-seq data. All competing methods are listed as follows:
\begin{itemize}
  \item ONMF\_$\ell_{20}$: ONMF with $\ell_{2,0}$-norm constraint (Eq. \ref{equ-7}).
  \item ONMF\_$\ell_{20}\_\rho$: ONMF\_$\ell_{20}$ with a fixed $\rho$ (Eq. \ref{equ-13}).
  \item ONMF\_$\ell_{c0}$: column-wise sparse ONMF (Eq. \ref{equ-22}).
  \item ONMF\_$\ell_0$: ONMF with $\ell_{0}$-norm constraint.
  \item NMF\_$\ell_{20}$: NMF with $\ell_{2,0}$-norm constraint.
  \item ONMF: Orthogonal NMF \cite{wang2019clustering}.
  \item NMF: Non-negative Matrix factorization \cite{xu2013block}.
  \item NMF\_$\ell_0$: NMF with $\ell_{0}$-norm constraint \cite{bo2014proximal}.
  \item NMF\_$\ell_{c0}$: Column-wise sparse NMF \cite{pock2016inertial}.
  \item Kmeans: A baseline unsupervised method.
  \item Kmeans\_$\ell_1$: Sparse $k$ means with $\ell_1$ penalty in \cite{witten2010framework}.
  \item Kmeans\_$\ell_0$: Sparse $k$ means with $\ell_\infty$/$\ell_0$ penalty in \cite{chang2018sparse}.
\end{itemize}
Wherein the proposed SSNMF methods in this study include
ONMF\_$\ell_{20}$,
ONMF\_$\ell_{20}\_\rho$,
ONMF\_$\ell_{c0}$,
ONMF\_$\ell_0$ and NMF\_$\ell_{20}$.

\subsection{Evaluation metrics}
To evaluate the clustering performance,
we use three metrics including NMI (Normalized Mutual Information) \cite{li2013clustering}, Purity and Entropy \cite{kim2007sparse}.
Suppose $c$ is the number of clustering clusters and $d$ is the number of true categories.
Then NMI is defined as
\begin{equation}\label{equ-26}
  \textrm{NMI} = \frac{\sum_{l=1}^{c} \sum_{h=1}^{d} t_{l,h} \log(\frac{n \cdot t_{l,h}}{t_l \hat{t}_h })}
                 {\sqrt{(\sum_{l=1}^c t_{l} \log\frac{t_{l}}{n})(\sum_{h=1}^{d} \hat{t}_h \log\frac{\hat{t}_h}{n})}},
\end{equation}
Purity is defined as
\begin{equation}\label{equ-27}
  \textrm{Purity} = \frac{1}{n} \sum_{l=1}^c \underset{1 \leq h \leq d}{\max} t_{l,h},
\end{equation}
Entropy is defined as
\begin{equation}\label{equ-28}
  \textrm{Entropy} = \frac{1}{n \log_2 d} \sum_{l=1}^c \sum_{h=1}^{d} t_{l,h} \log_2 \frac{t_{l,h}}{t_l},
\end{equation}
where $n$ is the number of considered samples, $t_l$ is the number of samples from the $l$-th cluster $C_l$, which is obtained by clustering method
and $\hat{t}_h$ is the number of samples from the $h$-th ground truth class $G_h$. $t_{l,h}$ denotes the number of overlapping samples between $C_l$ and $G_h$.
The larger the values of NMI and purity or the smaller the value of entropy, the better the clustering performance.

\subsection{Application to synthetic data}
We generate a synthetic data $\bm{X} \in \mathbb{R}^{p\times n}$ where $p = 500$ features and $n=60$ samples
from three true classes. 
Firstly, the elements in $\bm{X}$ satisfy $X_{ij} \sim N(0,1)$ ($1 \leq i \leq 60$, $1 \leq j \leq 20$),
$X_{ij} \sim N(0,1)$ ($31 \leq i \leq 90$, $21 \leq j \leq 40$)
and $X_{ij} \sim N(0,1)$ ($61 \leq i \leq 120$, $41 \leq j \leq 60$)
and $X_{ij} \sim 0.9*N(0,1)$  for the other $i$ and $j$, 
where $N(0,1)$ denotes standard normal distribution.
Secondly, we set $X_{ij} := |X_{ij}|$ for any $i$ and $j$ to ensure every element in the final synthetic data matrix $\bm{X}$ is positive (see Figure \ref{fig-2}A).

First of all, we show the convergence performance of PALM and maPALM when they are used to solve ONMF\_$\ell_{20}$ model in Eq. (\ref{equ-13}) on the synthetic data (Figure \ref{fig-2}B).
We find that the convergence speed of maPALM is significantly faster than that of PALM.

Secondly, to validate the effectiveness of our proposed methods to perform feature selection, we compare them with state-of-the-art methods on the synthetic data.
The parameters of these compared methods are carefully adjusted to give their best performances and all methods are repeated 10 times using different initial points for comparison.
The corresponding NMI scores are recorded for each method and the methods with higher NMI averages are regarded as more accurate ones. We find that the SSNMF methods including ONMF\_$\ell_{20}$ and NMF\_$\ell_{20}$ outperform other algorithms in terms of NMI (Figure \ref{fig-3}).
Interestingly, we also find that as $\rho$ becomes larger in the Algorithm \ref{alg-3}, the orthogonality level of estimated $\bm{H}$ is better, such that the clustering performance is better (Figure \ref{fig-4}). This result implies that it is possible to improve the clustering accuracy by adding the non-negative and orthogonal constraint in SSNMF model.

\begin{figure}[hbpt]
	\centering \includegraphics[width=1\linewidth]{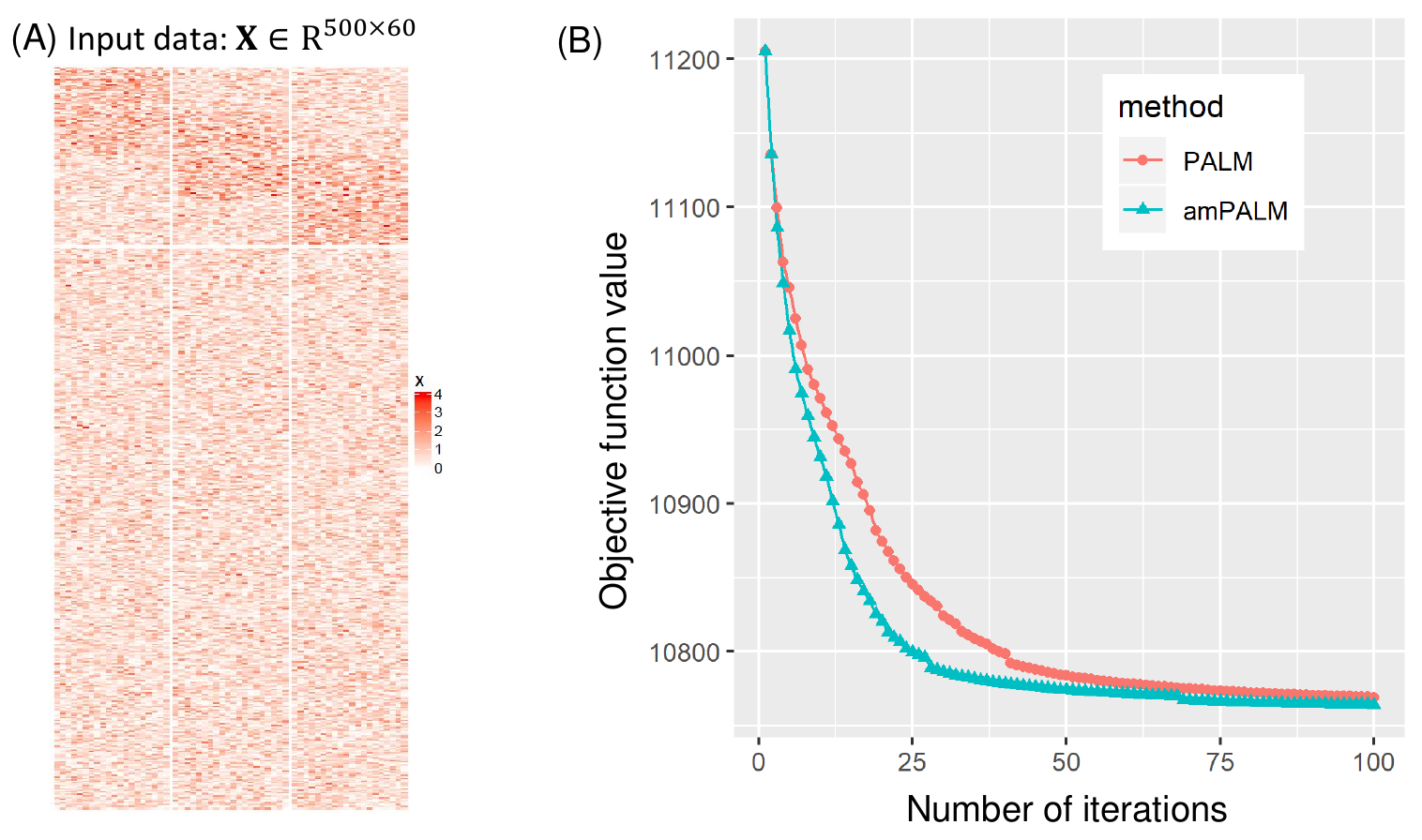}
	\caption{
		Results on the synthetic data.
		(A) Heatmap showing the synthetic data.
		(B) Convergence performance of PALM and maPALM for ONMF\_$\ell_{20}$ with a fixed $\rho=0.5$.
	}\label{fig-2}
\end{figure}

\begin{figure}[hbpt]
	\centering \includegraphics[width=1\linewidth]{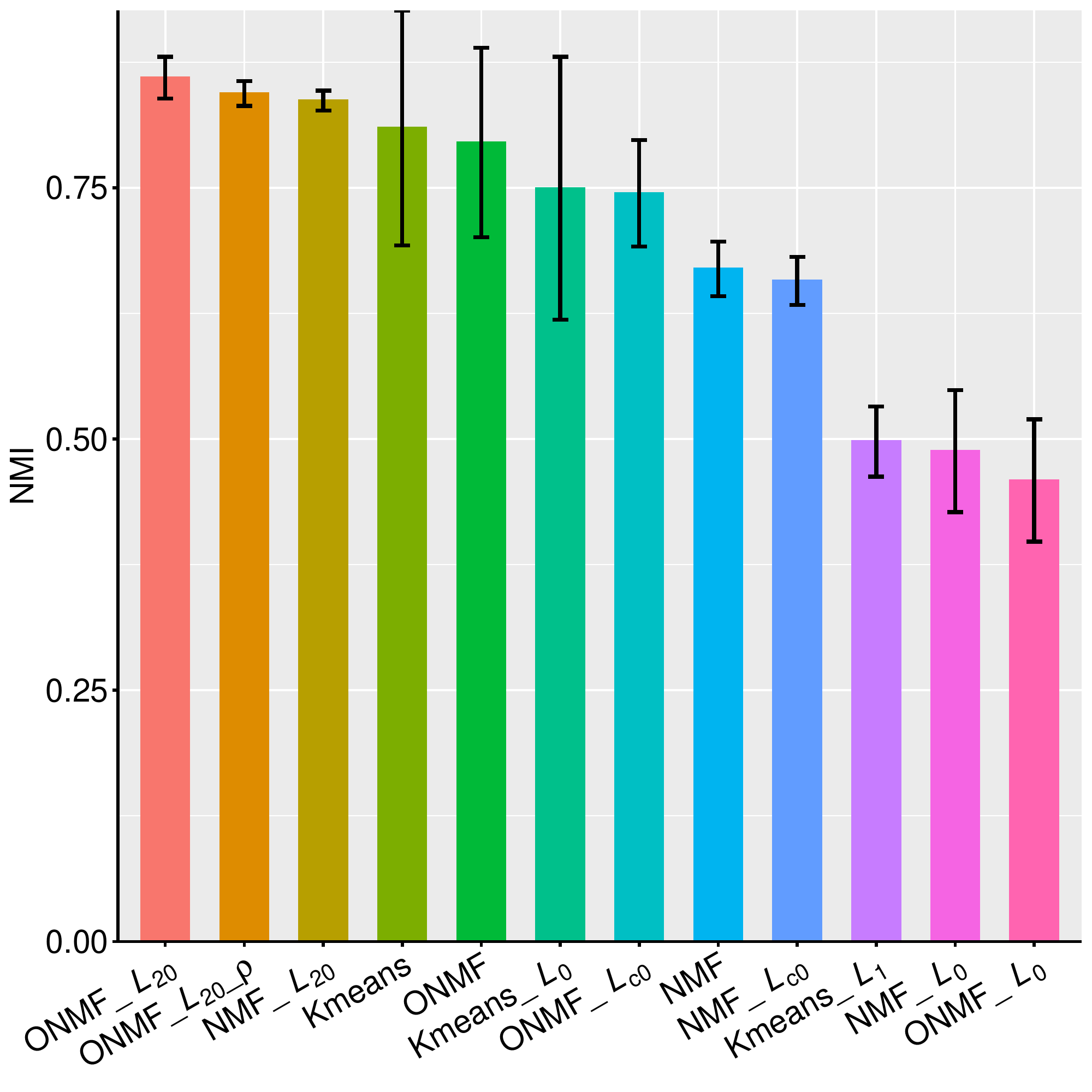}
	\caption{Comparison of 12 unsupervised clustering methods in terms of NMI on the synthetic data.
		Note that we set $\rho=1$ for ONMF\_$\ell_{20}\_\rho$.
	}\label{fig-3}
\end{figure}

\begin{figure}[hbpt]
	\centering \includegraphics[width=1\linewidth]{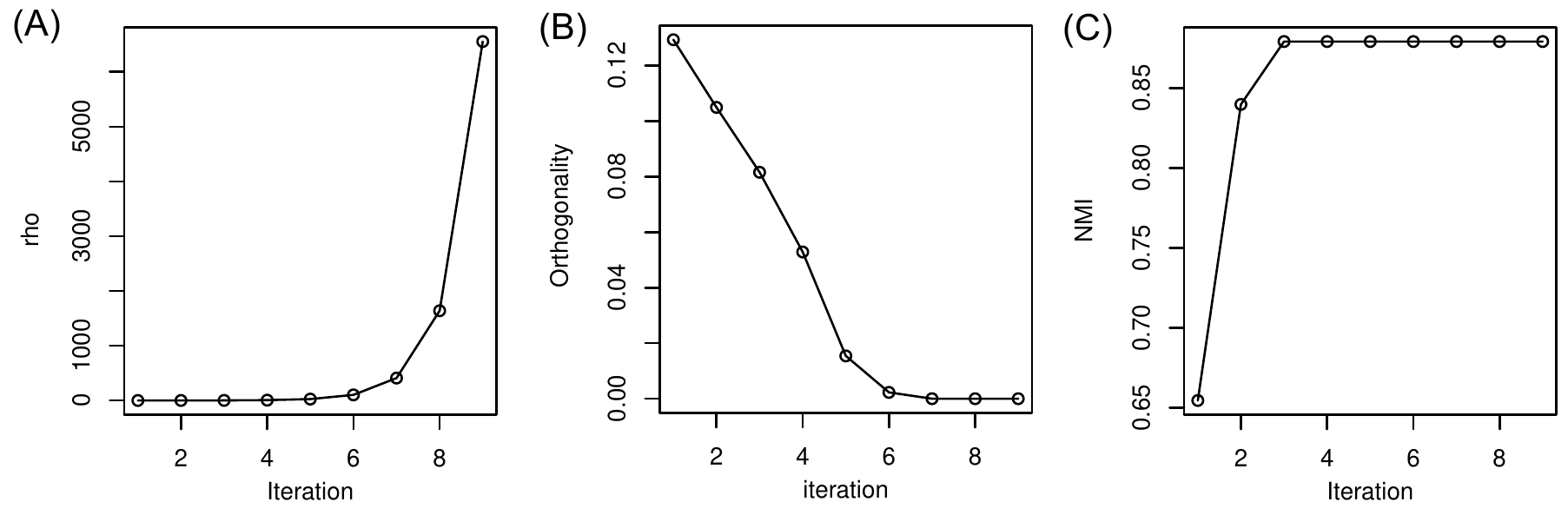}
	\caption{
		Scatter plots showing the change of (A) $\rho$ in Algorithm \ref{alg-3}, (B) orthogonality score of estimated $\bm{H}$ and (C) NMI score.
		For the definition of orthogonality level, please see Eq. 23 of \cite{wang2019clustering}.
	}\label{fig-4}
\end{figure}

Finally, we evaluate whether ONMF can detect outliers by sorting the values of estimated $\bm{H}$.
To this end, we generate a new synthetic data $\bm{X} \in \mathbb{R}^{p\times n}$ where $p = 500$ features and $n=60$ samples from three classes. Firstly, the elements in $\bm{X}$ satisfy
$X_{ij} \sim N(0,1)$ ($1 \leq i \leq 60$, $1 \leq j \leq 20$),
$X_{ij} \sim N(0,1)$ ($31 \leq i \leq 90$, $21 \leq j \leq 40$), and
$X_{ij} \sim 0.9*N(0,1)$  for other $i$ and $j$.
Secondly, we set $X_{ij} := |X_{ij}|$ and obtain the final synthetic data.
Note that the columns from 41 to 60 in $\bm{X}$ correspond to outliers.
We apply ONMF\_$\ell_{20}$ with parameter $k=90$ to the synthetic data.
We find that the outliers (\emph{i.e.}, noise samples) are those with relatively small values in the estimated $\bm{H}$ (Figure \ref{fig-5}). This result implies ONMF\_$\ell_{20}$ can detect these outliers by checking the values of columns in the estimated $\bm{H}$.

\begin{figure}[hbpt]
	\centering \includegraphics[width=1\linewidth]{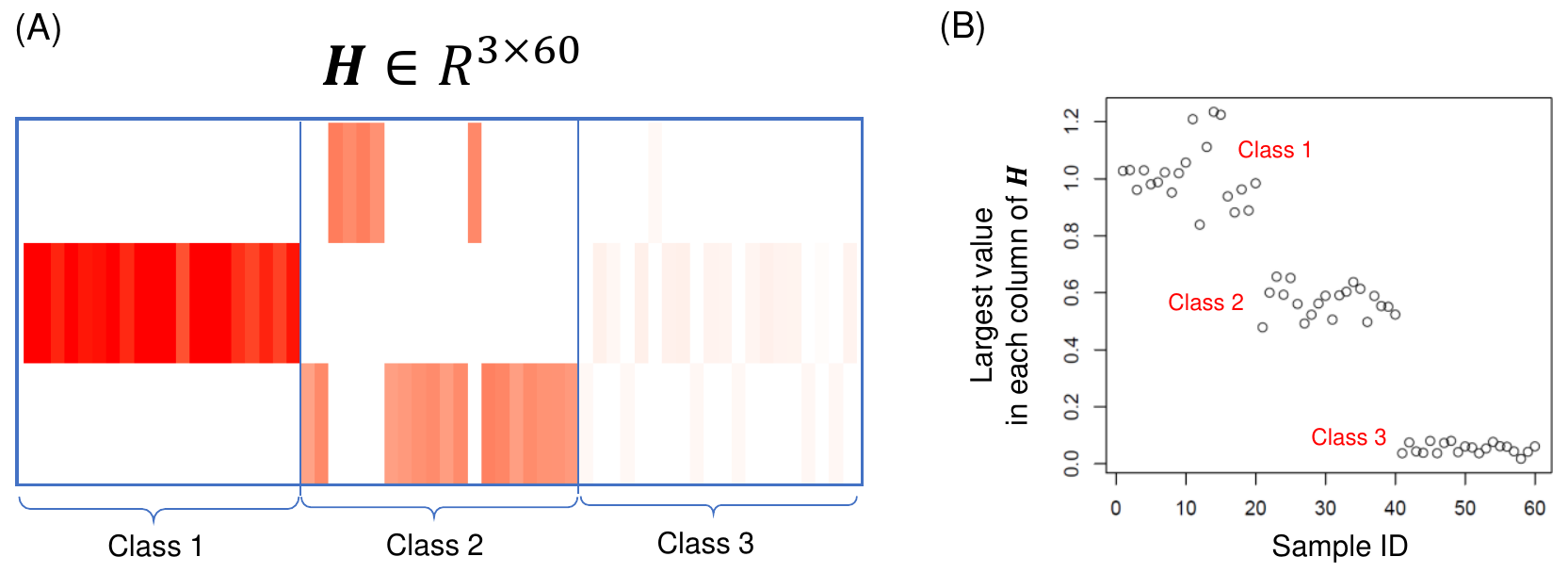}
	\caption{
		(A) Heatmap corresponding to the estimated $\bm{H}$ by ONMF\_$\ell_{20}$ on the second synthetic data.
		(B) Scatter plot showing the largest values from each column of the estimated $\bm{H}$.
	}\label{fig-5}
\end{figure}

\subsection{Application to biological data}
In this study, the proposed methods and other comparison methods are evaluated on three  scRNA-seq datasets as follows:
\begin{itemize}
  \item \textbf{Pollen} dataset \cite{pollen2014low} contains 301 single cells from 11 cell populations which are divided into 4 classes including Blood cells, Neural cells, Dermal or epidermal cells and Pluripotent cells.
  \item \textbf{Camp1} dataset \cite{camp2017multilineage} contains 425 single cells from human liver cells which are divided into 5 populations, named as iPS\_day\_0, De\_day\_6, IH\_day\_14, MH\_day\_21, and HE\_day\_8.
  \item \textbf{Lake} dataset \cite{lake2016neuronal} contains 3042 single cells from human liver cells which are divided into 16 populations named as Ex1, Ex3, Ex4, In6, In1, Ex5, In5, Ex7, In8, Ex2, Ex6, In7, In4, Ex8, In3, In2.
\end{itemize}
For each scRNA-seq dataset, we first filter out these genes which are not expressed in more than 70\% of cells. We then use a logarithmic transformation $x = \log2(x)$ to transform raw expression data. After preprocessing, 8747, 8058 and 5000 genes are retained for Pollen, Camp1 and Lake datasets, respectively.

\begin{table}[htbp]
\centering
\caption{Comparison in terms of (NMI \% $\pm$ std), (Purity \% $\pm$ std) and (Entropy \% $\pm$ std) on the Pollen, Lake and Camp1 datasets.}
\begin{adjustbox}{width=1\columnwidth,center}\label{tab-2}
\begin{tabular}{l|l|l|l|l}
  \hline
  \textbf{Pollen data}  & \#gene & NMI$\pm$sd &Purity$\pm$sd  &Entropy$\pm$sd \\
  \hline
  NMF\_$\ell_{20}$ & 2000 & \textbf{85.87 $\pm$ 1.29} & \textbf{91.86 $\pm$ 0.23} & \textbf{19.10 $\pm$ 0.97} \\
  ONMF\_$\ell_{20}\_\rho$ & 2000 & 84.96 $\pm$ 0.86 & 91.69 $\pm$ 0.16 & 19.79 $\pm$ 0.61 \\
  ONMF\_$\ell_{20}$ & 2000 & 80.71 $\pm$ 3.56 & 89.30 $\pm$ 2.76 & 24.34 $\pm$ 4.30 \\
  ONMF\_$\ell_{c0}$ &*  & 81.62 $\pm$ 0.32 & 90.73 $\pm$ 0.11 & 22.28 $\pm$ 0.22 \\
  ONMF\_$\ell_{0}$ &*  & 51.67 $\pm$ 4.58 & 67.08 $\pm$ 5.65 & 27.56 $\pm$ 1.92 \\
  NMF\_$\ell_{0}$ &*  & 54.46 $\pm$ 6.86 & 67.34 $\pm$ 5.78 & 24.97 $\pm$ 5.59 \\
  NMF\_$\ell_{c0}$ &*  & 81.45 $\pm$ 0.88 & 91.89 $\pm$ 0.17 & 21.28 $\pm$ 0.69 \\
  ONMF & all & 83.70 $\pm$ 2.81 & 91.23 $\pm$ 2.52 & 20.52 $\pm$ 4.07 \\
  Kmeans & all & 71.39 $\pm$ 7.89 & 82.52 $\pm$ 5.27 & 31.37 $\pm$ 5.17 \\
  Kmeans$\_\ell_{0}$ & 2000 & 72.36 $\pm$ 7.30 & 81.40 $\pm$ 8.25 & 29.48 $\pm$ 7.92 \\
  Kmeans$\_\ell_{1}$ & 2039 & 78.09 $\pm$ 9.60 & 88.54 $\pm$ 7.01 & 21.45 $\pm$ 5.43 \\
  \hline
  \textbf{Camp1 data}  & \#gene & NMI$\pm$sd &Purity$\pm$sd  &Entropy$\pm$sd \\
  \hline
  NMF\_$\ell_{20}$ & 2000 & \textbf{89.94} $\pm$ 1.30 & \textbf{96.07 $\pm$ 0.39} & 9.88 $\pm$ 0.88 \\
  ONMF\_$\ell_{20}\_\rho$ & 2000 & 89.32 $\pm$ 3.16 & 92.40 $\pm$ 8.13 & 10.96 $\pm$ 3.68 \\
  ONMF\_$\ell_{20}$ & 2000 & 71.00 $\pm$ 9.25 & 72.12 $\pm$ 5.54 & 29.04 $\pm$ 6.15 \\
  ONMF\_$\ell_{c0}$ &*  & 47.31 $\pm$ 8.06 & 44.35 $\pm$ 5.13 & 10.25 $\pm$ 3.95 \\
  ONMF\_$\ell_{0}$  &*  & 18.57 $\pm$ 18.69 & 34.14 $\pm$ 7.56 & \textbf{5.85 $\pm$ 4.99} \\
  NMF\_$\ell_{0}$   &*  & 12.41 $\pm$ 17.04 & 31.25 $\pm$ 6.72 & 5.94 $\pm$ 6.38 \\
  NMF\_$\ell_{c0}$  &*  & 85.54 $\pm$ 5.20 & 91.67 $\pm$ 2.74 & 17.27 $\pm$ 4.10 \\
  ONMF & all & 78.98 $\pm$ 9.88 & 78.31 $\pm$ 9.89 & 23.19 $\pm$ 8.71 \\
  Kmeans & all & 78.06 $\pm$ 1.18 & 79.72 $\pm$ 5.76 & 24.09 $\pm$ 4.16 \\
  Kmeans$\_\ell_{0}$ & 2000 & 73.13 $\pm$ 5.36 & 71.62 $\pm$ 7.18 & 28.67 $\pm$ 4.58 \\
  Kmeans$\_\ell_{1}$ & 2005 & 71.96 $\pm$ 1.32 & 79.11 $\pm$ 1.10 & 28.07 $\pm$ 1.33 \\
  \hline
  \textbf{Lake data}  & \#gene & NMI$\pm$sd &Purity$\pm$sd  &Entropy$\pm$sd \\
  \hline
  NMF\_$\ell_{20}$ & 2000 & \textbf{72.43 $\pm$ 2.15} & \textbf{76.73 $\pm$ 2.25} & 28.88 $\pm$ 2.86 \\
  ONMF\_$\ell_{20}$ & 2000 & 57.19 $\pm$ 4.16 & 63.92 $\pm$ 4.30 & 42.67 $\pm$ 5.01 \\
  ONMF\_$\ell_{20}\_\rho$ & 2000 & 71.99 $\pm$ 2.19 & 76.23 $\pm$ 2.49 & 29.60 $\pm$ 2.60 \\
  ONMF\_$\ell_{c0}$ &*  & 10.23 $\pm$ 10.54 & 36.97 $\pm$ 2.42 & 8.33 $\pm$ 8.70 \\
  ONMF\_$\ell_{0}$ &*  & 0.92 $\pm$ 0.64 & 35.01 $\pm$ 0.21 & \textbf{1.53 $\pm$ 1.42} \\
  NMF\_$\ell_{0}$ &*  & 17.08 $\pm$ 14.82 & 41.30 $\pm$ 6.59 & 11.05 $\pm$ 11.93 \\
  NMF\_$\ell_{c0}$ &*  & 66.35 $\pm$ 2.20 & 72.00 $\pm$ 1.95 & 34.61 $\pm$ 3.22 \\
  ONMF & all & 61.44 $\pm$ 3.84 & 67.08 $\pm$ 4.00 & 40.20 $\pm$ 4.43 \\
  Kmeans & all & 59.96 $\pm$ 1.31 & 68.43 $\pm$ 1.68 & 42.10 $\pm$ 1.59 \\
  Kmeans$\_\ell_{0}$ & 2000 & 62.05 $\pm$ 2.67 & 68.93 $\pm$ 2.62 & 40.01 $\pm$ 2.91 \\
  Kmeans$\_\ell_{1}$ & 2270 & 70.73 $\pm$ 1.34 & 77.62 $\pm$ 2.12 & 32.32 $\pm$ 1.28 \\
  \hline
\end{tabular}
\end{adjustbox}
\end{table}
\begin{figure*}[hbpt]
	\centering \includegraphics[width=1\linewidth]{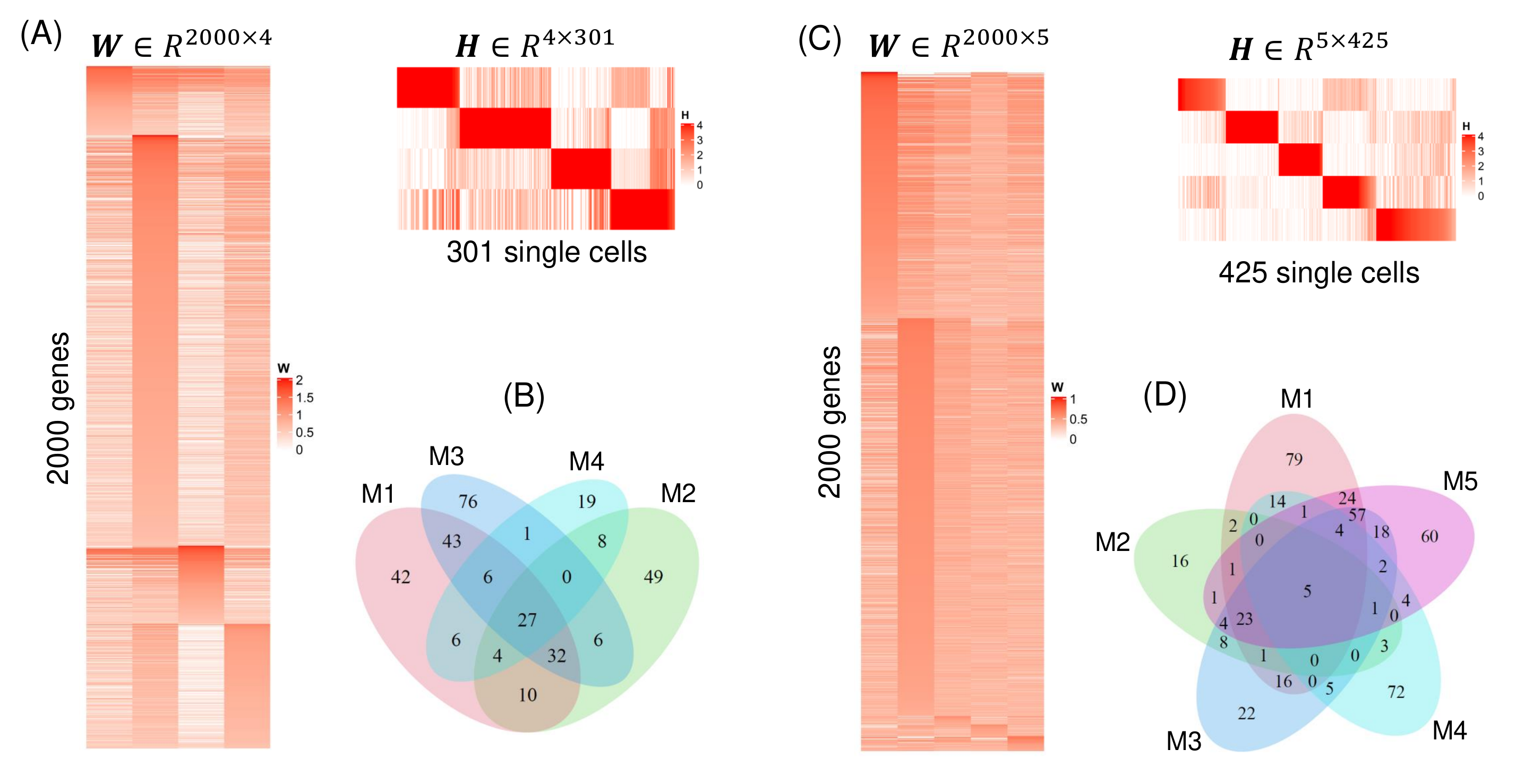}
	\caption{Results are shown in (A) to (D) when NMF\_$\ell_{20}$ is applied to the Pollen and Camp1 scRNA-seq datasets.
    (A) Heatmap showing $\bm{W}$ and $\bm{H}$ obtained from NMF\_$\ell_{20}$ with $k=2000$ on the Pollen dataset where 2000 genes are selected across 301 cells.
    (B) A Venn diagram showing overlap level for the selected genes from four biclusters (M1 to M4) on the Pollen dataset and total of 329 genes are selected.
    (C) Heatmap showing $\bm{W}$ and $\bm{H}$ obtained from NMF\_$\ell_{20}$ with $k=2000$ on the Camp1 dataset where 2000 genes are selected across 425 cells.
    (D) A Venn diagram showing overlap level for the selected genes from five biclusters (M1 to M5) on the Camp1 dataset and total of 443 genes are selected.
	}\label{fig-6}
\end{figure*}

For ONMF\_$\ell_{20}$, we set $k = 2000$ (to extract 2000 genes) which is to ensure that the number of selected genes is about 2000 for further analysis of biological function; rank $r$ equal to the number of true classes of scRNA-seq datasets (herein $r=4$ for the Pollen dataset, $r=5$ for the Camp1 dataset and $r=16$ for the Lake dataset); $\rho = 0.1$ and $\gamma = 1.5$ which are for updating next $\rho := \gamma \rho$ (herein $\rho$ is updated up to 10 times); and $\epsilon = 1e-3$ in Algorithm \ref{alg-3}. For fairness of comparison, we ensure these estimated $\bm{W}$ of all sparse learning methods including NMF\_$\ell_{20}$, ONMF\_$\ell_{c0}$, NMF\_$\ell_{c0}$, ONMF\_$\ell_{0}$, NMF\_$\ell_{0}$ and Kmeans$\_\ell_{0}$ have the same sparsity level, and all methods are repeated 10 times using different initial points for comparison.

We evaluate the clustering performance of all methods in terms of NMI, Purity and Entropy. The detailed results on the pollen, camp1 and lake datasets are summarized in Table \ref{tab-2}. These results show that the proposed SSNMF methods, especially NMF\_$\ell_{20}$, outperform other methods. The use of $\ell_{2,0}$-norm enables some SSNMF methods to select some important features by checking the non-zero rows of $\bm{W}$. Finally, we also discuss the influence of $k$-choice on the clustering performance for these proposed SSNMF with $\ell_{2,0}$-norm methods and the results show that the proposed NMF\_$\ell_{20}$  outperforms Kmeans\_$\ell_{0}$ in different situations (Table \ref{tab-3}).

\begin{table}[ht]
	\centering
	\caption{Comparison of (NMI \%) averages with different number of genes on the Pollen, Camp1 and Lake datasets.}\label{tab-3}
	\begin{tabular}{lllllll}
		\hline
		\textbf{Pollen data} ($k$=)& 500 & 1000 & 2000 & 3000 & 4000 & 5000 \\
		\hline
		
		NMF\_$\ell_{20}$        & 81.77 & 84.46 & 85.87 & 85.70 & 85.04 & 85.04 \\
		ONMF\_$\ell_{20}$       & 69.36 & 73.06 & 80.71 & 84.27 & 84.65 & 84.34 \\
		ONMF\_$\ell_{20}\_\rho$ & 82.35 & 84.09 & 84.96 & 86.02 & 85.21 & 85.04 \\
		Kmeans$\_\ell_{0}$      & 80.79 & 76.98 & 72.36 & 71.04 & 71.57 & 69.49 \\
		
		\hline
		\textbf{Camp1 data} ($k$=)& 500 & 1000 & 2000 & 3000 & 4000 & 5000 \\
		\hline
		
		NMF\_$\ell_{20}$        & 91.56 & 91.84 & 89.94 & 90.11 & 90.28 & 90.28 \\
     	ONMF\_$\ell_{20}$       & 62.57 & 63.93 & 71.00 & 72.54 & 77.34 & 80.77 \\
		ONMF\_$\ell_{20}\_\rho$ & 90.59 & 90.80 & 89.32 & 89.21 & 89.24 & 89.03 \\
		Kmeans$\_\ell_{0}$      & 74.26 & 73.41 & 73.13 & 76.28 & 75.31 & 74.23 \\
		
		\hline
	    \textbf{Lake data} ($k$=)& 500 & 1000 & 2000 & 3000 & 4000 & 5000 \\
		\hline
		NMF\_$\ell_{20}$        & 72.96 & 72.88 & 72.43 & 72.34 & 72.29 & 72.23 \\
		ONMF\_$\ell_{20}$       & 52.94 & 55.67 & 57.19 & 58.64 & 59.95 & 61.44 \\
		ONMF\_$\ell_{20}\_\rho$ & 72.48 & 72.37 & 71.99 & 71.67 & 71.55 & 71.52 \\
		Kmeans$\_\ell_{0}$      & 70.83 & 69.25 & 62.05 & 58.83 & 56.17 & 55.65 \\
		\hline
	\end{tabular}
\end{table}

\subsection{Biological analysis}
In this section, we show that the SSNMF with $\ell_{2,0}$-norm methods can be used for the identification of subpopulation and gene selection for scRNA-seq data. Based on the experiment results in the previous section, the NMF\_$\ell_{20}$ achieved the best performance.
Therefore, we select the computing results from NMF\_$\ell_{20}$ ($k=2000$) on Pollen and Camp1 datasets as a example for further biological analysis. Some of the significant gene and sample expression patterns can be identified based on the output $\bm{W}\in R^{p\times r}$ and $\bm{H}\in R^{r\times n}$ from NMF\_$\ell_{20}$.

To be simplify, a biological functional bicluster is defined as a gene subset with a sample subset (also as cell subset). For each pair of $\bm{w}_i$ (in $\bm{W}$) and $\bm{h}^i$ (in $\bm{H}$), a bicluster is extracted based on the following computational steps.
\begin{itemize}
	\item Step 1: For the $i$-th column of $\bm{W}$ ($\bm{w}_i$), the higher numerical values, the more important the corresponding genes are. We do a z-score normalization for the $\bm{w}_i$ using the formula $\bm{z}_i = (\bm{w}_i-mean(\bm{w}_i))/sd(\bm{w}_i)$. The genes with the corresponding coefficient values of $\bm{w}_i$, whose z-scores are larger than a given threshold $T$, are extracted as the bicluster gene subset.
	
	\item Step 2: For $i$-th row of $\bm{H}$ ($\bm{h}^i$), we screen the cells with the largest coefficient in their corresponding column, i.e., 
	$\{j|w_{ij} \geq w_{tj}~\mbox{for}~j=1,\cdots,n,~\forall~t\}$. Thus, we obtain a single cell set for the bicluster.
\end{itemize}

For the output $\bm{W}$ and $\bm{H}$ from Pollen dataset, we first calculate z-scores normalization for each column of $\bm{W}$ and then rank all genes according to their z-socre values.
These genes with z-score more than the threshold $T = 1.5$ are regarded as the gene set of biclusters.
We extract four biclusters with total 329 genes (Figure \ref{fig-6}A and B).
Bicluster 1 contains 170 genes and 68 cells. The selected cells are all Blood cells.
Bicluster 2 contains 136 genes and 99 cells and all selected cells are Dermal/Epidermal cells.
Bicluster 3 contains 191 genes and 65 cells and all selected cells are Neural cells.
Bicluster 4 contains 71 genes and 69 cells while 45 of 69 cells are Blood cells and 24 of 69 are Pluripotent cells.
To be interesting, some genes are shared in multiple biclusters. 27 genes are shared on all four biclusters. It shows that the genes may be hub genes and play a joint role in multiple biological functions (pathways).

To demonstrate the biological function of the gene sets from these identified biclusters, 
we perform the gene function enrichment analysis. 
We retrieve the KEGG pathways data information from Molecular Signatures Database (MSigDB, http://www.gsea-msigdb.org/gsea/msigdb/index.jsp). KEGG pathways are a class of collection of manually drawn pathway maps representing the biological knowledge of the molecular interaction, reaction and relation networks. Each KEGG pathway is consisted by a set of functionally gene set which expresses complex regulatory mechanism among different genes. The hypergeometric test is applied for the biological statistical analysis. A class of critical KEGG signaling pathway is significantly enriched for the identified four bicluster gene sets with the Benjamini-Hochberg adjusted $p < 0.05$.

As expected, most of the enriched KEGG signaling pathways are related on the developing cerebral cortex which is highly consistent to the  Pollen dataset from the diverse neural cell types \cite{pollen2014low}.
Interestingly, we find that multiple bicluster gene sets are enriched on the ribosome pathway which is the cell factories responsible for making proteins and the ribosome pathway has been reported to play an important role in brain development \cite{chau2018downregulation}.
In addition, some brain disease-related pathways have been discovered.
For example, 
bicluster 1 is enriched in the pathways like Parkinson's disease ($p=$7.5e-14) and Alzheimer's disease ($p=$3.9e-11).
Bicluster 2 is enriched in some KEGG pathways like proteasome ($p=$1.2e-3), leukocyte transendothelial migration ($p=$1.2e-3), regulation of actin cytoskeleton ($p=$1.2e-3), and focal adhesion ($p=$1.4e-3).

For the Camp1 dataset, five biclusters are extracted from the $\bm{W}$ and $\bm{H}$ (Figure \ref{fig-6}C and D).
Similarly, we calculate z-scores normalization for each column of $\bm{W}$ and then rank all genes according to their z-socre values.
These genes with z-score more than the threshold $T = 1.2$ are regarded as the gene set of biclusters.
Bicluster 1 contains 227 genes and 73 cells where all selected cells are MH day 21 cells, 
bicluster 2 contains 65 genes and 81 cells where 80 of 81 cells are iPS day 0 cells, 
bicluster 3 contains 166 genes and 67 cells where all selected cells are De day 6 cells, 
bicluster 4 contains 111 genes and 82 cells where 77 of 82 cells are IH day 14 cells and 5 of 82 are MH day 21 cells, and 
bicluster 5 contains 205 genes and 122 cells where 113 of 122 cells are HE day 8 cells, 3 of 122 are De day 6 cells, 4 of 122 are IH day 14 cells and 2 of 122 are MH day 21 cells.
Similar to the results on Pollen dataset, the cells in each identified bicluster are highly pure of population. Meanwhile, each bicluster possesses different domain genes while a small number of genes overlap each other (see Figure \ref{fig-6}D).
Due to these extracted single cells in the Camp1 dataset are from human liver \cite{camp2017multilineage}, the biological function analysis show that the identified bicluster gene sets are significantly enriched a series of KEGG pathways (Benjamini-Hochberg adjusted $p < 0.05$) which are highly related to biological processes associated with liver.
For example, bicluster 1 is significantly enriched in the oxidative phosphorylation pathway ($p = 7.6e-13$ and there are 26 genes in this pathway) which have been reported to be related to liver \cite{santacatterina2016down}.
Bicluster 2 is significantly enriched in the glycolysis gluconeogenesis pathway ($p = 6.2e-5$ and there are 6 genes in this pathway) which have been reported to be related to liver \cite{ma2013switch}.
All these results show that the proposed method can be used for single cell type discovery, gene selection and biological process analysis.

\section{Conclusion}
In this paper, we present a class of SSNMF models with $\ell_{2,0}$-norm constraint.
We prove that $\ell_{2,0}$-norm satisfies the K\L~property,
such that the PALM algorithm can be used to solve a class of non-convex and non-smooth optimization problems with $\ell_{2,0}$-norm constraint.
Especially, we first introduce the NMF\_$\ell_{20}$ model which integrates feature selection in the NMF model. To improve the convergence rate of PALM, we further develop an accelerated version of PALM (maPALM) to solve NMF\_$\ell_{20}$. We also prove the convergence of proposed algorithms (PALM and maPALM) when they are used to solve NMF\_$\ell_{20}$.
To integrate feature selection and non-negative orthogonal constraint in the NMF model, we furthermore introduce the ONMF\_$\ell_{20}$ model.
We develop an efficient algorithm to solve it by using a penalty function method.
Briefly, the algorithm converts ONMF\_$\ell_{20}$ into a series of constrained and penalized NMF problems which can be solved by the PALM and maPALM algorithms.
Finally, we compare these proposed SSNMF methods with other methods for clustering task on the synthetic and scRNA-seq data. The results show that the proposed SSNMF methods can be used not only for clustering (single cell type discovery), but also for gene selection and biological function analysis.

\appendices
\section{Definitions and Proofs}\label{appendix-a}

\subsection{Mathematical definitions for non-convex analysis}\label{appendix-a1}
We introduce some mathematical definitions which are used in this study for non-convex analysis \cite{liu2019con,bo2014proximal,bao2015dictionary}.
\begin{defn}
(Proper) $f(\bm{x})$ is proper if $\mathbf{dom}(f) := \{\bm{x} \in \mathbb{R}^n: f(\bm{x})<+\infty \}$ is nonempty and $f(\bm{x})> -\infty$.
\end{defn}
\begin{defn}
(Lower semi-continuous) $f(\bm{x})$ is lower semi-continuous if $\liminf\limits_{\bm{x}\rightarrow \bm{x}_0} f(\bm{x}) \geq f(\bm{x}_0)$ at any point $\bm{x}_0 \in \mathbf{dom}(f)$.
\end{defn}
\begin{defn}
(Coercive Function) $f(\bm{x})$ is called coercive if $f(\bm{x})$ is bounded from below and $f(\bm{x}) \rightarrow \infty$ if $\|\bm{x}\| \rightarrow \infty$.
\end{defn}
\begin{defn}
 (Lipschitz smooth) $f(\bm{x})$ is Lipschitz smooth if it is differentiable and there exists $L>0$ and such that
$$\|\nabla f(\bm{x})  - \nabla f(\bm{y})\| \leq L \|\bm{x} - \bm{y}\|, \forall \bm{x}, \bm{y} \in \mathbb{R}^n.$$
Any such $L$ is considered to as a Lipschitz constant for $f(\bm{x})$.
\end{defn}
\begin{defn}
(Subdifferential)
Let $f: \mathbb{R}^n \rightarrow (-\infty, \infty]$ be a proper and lower semi-continuous function. Then the Frecht sub-differential
of $f$, denoted as $\widehat{\partial} f$, at point $\bm{x} \in \mathbf{dom}(f)$ is the set of all vectors $\bm{z}$ which satisfies
$$\liminf\limits_{\bm{x}\neq \bm{y}, \bm{y} \rightarrow \bm{x}} \frac{f(\bm{y})-f(\bm{x}) - \langle \bm{z}, \bm{y}-\bm{x} \rangle}{\|\bm{y}-\bm{x} \|} \geq 0,$$
where $\langle\cdot,\cdot\rangle$ denotes the inner product.
Then the limiting Frecht sub-differential, or simply the sub-differential, denoted as $\partial f$, at $\bm{x} \in \mathbf{dom}(f)$ is the following closure of $\widehat{\partial} f$:
$$\{ \bm{z} \in \mathbb{R}^n: \exists (\bm{x}^k, g(\bm{x}^k)) \rightarrow (\bm{x}, f(\bm{x})) \},$$
where $\bm{z}^k \in \widehat{\partial} f(\bm{x}^k) \rightarrow \bm{z}$ when $k \rightarrow \infty$.
\end{defn}
\begin{defn}
(Critical Point) A point $\bm{x}$ is called a critical point of function $f$ if $0 \in \partial f(\bm{x})$.
\end{defn}

\begin{defn}
(Semi-algebraic set and function)
A subset $\Omega$ of $\mathbb{R}^n$ is a real semi-algebraic set if there exist a finite number of real polynomial functions $r_{ij}, h_{ij}: \mathbb{R}^n \rightarrow \mathbb{R}$ such that
$$\Omega = \bigcup_{j=1}^{p} \bigcap_{i=1}^{q} \big\{ \bm{x}\in \mathbb{R}^n: r_{ij}(\bm{x})=0~ \mbox{and}~h_{ij}(\bm{x})<0 \big\}.$$
Function $f$ is called semi-algebraic function if its graph
$\{(\bm{x},z) \in \mathbb{R}^{n+1}: f(\bm{x})=z\}$ is a semi-algebraic subset of $\mathbb{R}^{n+1}$.
\end{defn}

\begin{defn}\label{defn-12}
(Kurdyka-\L{ojasiewicz} property and function)
Function $f$ has the Kurdyka-\L{ojasiewicz} (K\L) property at $\bar{\bm{x}} \in \mathbf{dom}(\partial f) := \{ \bm{x} \in \mathbb{R}^{n}: \partial f(\bm{x}) \neq \emptyset \}$
if there exist $\eta \in (0, \infty]$, a neighborhood $U$ of $\bar{\bm{x}}$ and a function
$\phi: [0, \eta) \rightarrow \mathbb{R}_+$ which satisfies (1) $\phi$ is continuous
at 0 and $\phi(0)=0$; (2) $\phi$ is concave and $C^1$ on $(0, \eta)$; (3) for all $s \in (0, \eta): \phi'(s)>0$, such that for all
$\bm{x} \in U \cap [f(\bar{\bm{x}}) < f(\bm{x}) < f(\bar{\bm{x}})+\eta],$
the following inequality holds
$$ \phi'(f(\bm{x}) - f(\bar{\bm{x}})) dist(0, \partial f(\bm{x}))\geq 1.$$
Function $f$ is called a K\L~function if $f$ satisfies the K\L~property at each point of $\mathbf{dom}(\partial f)$.
Moreover, Theorem 3 in \cite{bo2014proximal} shows that all proper, lower semicontinuous and semi-algebraic functions satisfy K\L~property property.
\end{defn}

\subsection{Proof of Proposition \ref{prop-1}}\label{appendix-a2}
\begin{proof}
Reference \cite{wang2019clustering}, we give the following proof.
Let $F_{\rho}(\bm{W},\bm{H}) =  \frac{1}{2} \|\bm{X} - \bm{W}\bm{H}\|_F^2 + \frac{\rho}{2}\sum_{j=1}^n \Big((\bm{1}^T\bm{h}_j)^2 - \|\bm{h}_j\|_2^2\Big)$, $\bm{Z}^* := (\bm{W}^*,\bm{H}^*)$ is a local minimizer of  (\ref{equ-10}). If $\Omega_1$ is the feasible set of (\ref{equ-7}), then there exist a neighborhood of $\bm{Z}^*$, $N_\epsilon(\bm{Z}^*)=\{\bm{Z}|\| \bm{Z} - \bm{Z}^* | \}$ where $\epsilon>0$ and it satisfies $F_\rho(\bm{Z}^*) \leq F_\rho(\bm{Z})$ for any $\bm{Z} \in N_\epsilon(\bm{Z}^*) \cap \Omega$.
Let $\phi(\bm{h}_j) = (\bm{1}^T\bm{h}_j)^2 - \|\bm{h}_j\|_2^2$, we note that $\phi(\bm{h}_j)>0$ if $\bm{h}_j$ has at least two non-zeros entries and $\phi(\alpha \bm{h}_j)$ is an increasing function with respect to $\alpha$. Suppose that there exists an $j'$ and $\bm{h}_{j'}^*$ is infeasible to (\ref{equ-7}) and thus $\phi(\bm{h}_{j'}^*)>0$.
Then, for a scalar $\alpha\in(0,1)$ and $\bm{Z}_\alpha := (\bm{W}^*/\alpha, \alpha\bm{H}^*)$ is a feasible point to (\ref{equ-10}).
Specifically, since $\|\bm{Z}_\alpha -\bm{Z}^*\| = (1/\alpha - 1)^2\|\bm{W}^*\|_F^2 + (\alpha - 1)^2 \|\bm{H}^*\|_F^2 \leq \max\{(1/\alpha - 1)^2, (\alpha - 1)^2\} \|\bm{Z}^*\|_F^2$.
To have $\bm{Z}_\alpha \in N_\epsilon(\bm{Z}^*)$, it is sufficient to let
$\max\{(1/\beta - 1)^2, (\beta - 1)^2\} < \alpha < 1$
where $\beta = \sqrt{\epsilon/(2*\|\bm{Z}^*\|_F^2)}$.
Thus, we have $F_\rho(\bm{Z}_\alpha) -  F_\rho(\bm{Z}^*) = \phi(\alpha \bm{h}_{j'}^*) - \phi(\bm{h}_{j'}^*) < 0.$
This is a contradiction. So, we have proved that $(\bm{W}^*,\bm{H}^*)$  is a feasible solution of (\ref{equ-7}).
Next, will prove that $(\bm{W}^*,\bm{H}^*)$ is also a local minimizer of (\ref{equ-7}).
Let $\Omega_2$ as the feasible set of (\ref{equ-7}), then we have $\{N_\epsilon(\bm{Z}^*) \cap \Omega_2\} \subseteq \{N_\epsilon(\bm{Z}^*) \cap \Omega_1\}$.
This implies that $G_\rho(\bm{Z}^*) = F_\rho(\bm{Z}^*) \leq  F_\rho(\bm{Z}) = G_\rho(\bm{Z})$ for any $\bm{Z} \in N_\epsilon(\bm{Z}^*) \cap \Omega_2$ where $G_{\rho}(\bm{W},\bm{H}) =  \frac{1}{2} \|\bm{X} - \bm{W}\bm{H}\|_F^2$. So, we prove that $(\bm{W}^*,\bm{H}^*)$ is a local minimizer of (\ref{equ-7}).
\end{proof}

\section*{Acknowledgment}
This work was supported by 
Key-Area Research and Development Program of Guangdong Province [2020B0101350001], and
the National Science Foundation of China [61272274],
and Natural Science Foundation of Jiangxi Province of China [20192BAB217004] and China Postdoctoral Science Foundation [2020M671902].

\balance
\small{
\bibliographystyle{IEEEtran}
\bibliography{MYREF}
}

\end{document}